\documentclass[11pt,bookmarks=false]{article}

\usepackage[colorlinks,citecolor=blue,urlcolor=blue]{hyperref}
\usepackage{amssymb}
\usepackage{amsthm}
\usepackage{amsmath}
\usepackage{booktabs}
\usepackage{mathrsfs}
\usepackage{multirow}
\usepackage{bm}
\usepackage{geometry}
\usepackage{graphics}
\usepackage{float}
\usepackage{epsfig}
\usepackage{subfigure}
\usepackage{enumerate}

\numberwithin{equation}{section}
\newtheorem{thm}{Theorem}[section]
\newtheorem{defn}{Definition}[section]
\newtheorem{coro}[thm]{Corollary}
\newtheorem{lemma}[thm]{Lemma}
\newtheorem{prop}[thm]{Proposition}

\def\proclaim#1{\par \bigskip\noindent {\bf #1}\bgroup\it\ }
\def\endproclaim{\egroup\par\bigskip}

\newcommand{\dif}{\mathrm{d}}

\def\text#1{\mbox{\rm #1}}
\def\overset#1#2{\stackrel{#1}{#2} }

\begin{document}
\title{\bf Arena Model: Inference About Competitions}
\date{}
\maketitle
\begin{center}
\vskip -1.5cm {{\sc {\large Chenhe Zhang}\footnote[1]{Department of Mathematics, Zhejiang University, Hangzhou 310007, P.R.China; 3150104161@zju.edu.cn}  and  {\large Peiyuan Sun}\footnote[2]{Department of Mathematics, Zhejiang University, Hangzhou 310007, P.R.China; 3150104069@zju.edu.cn}}}
\end{center}

\bigskip
\noindent{\bf Abstract.} 
The authors propose a parametric model called the arena model for prediction in paired competitions, i.e. paired comparisons with eliminations and bifurcations. The arena model has a number of appealing advantages. First, it predicts the results of competitions without rating many individuals. Second, it takes full advantage of the structure of competitions. Third, the model provides an easy method to quantify the uncertainty in competitions. Fourth, some of our methods can be directly generalized for comparisons among three or more individuals. Furthermore, the authors identify an invariant Bayes estimator with regard to the prior distribution and prove the consistency of the estimations of uncertainty. Currently, the arena model is not effective in tracking the change of strengths of individuals, but its basic framework provides a solid foundation for future study of such cases.

\noindent{\bf AMS 2010 subject classification:} 60K37, 62F07, 62F15.
\\
\noindent{\bf Keywords:} paired comparisons, competitions, arena model without fluctuations, Bayesian inference, arena model with fluctuations, the coefficient of fluctuations.

\section{Introduction}
The research on paired comparisons has a long history. In 1927, Thurstone \cite{PA} studied a psychological continuum and compared two physical stimulus magnitudes. Two decades later, Bradley and Terry \cite{RAI} proposed a model for rating players and Elo \cite{RCP} developed a system with a heuristic algorithm to update ranks of the players. In the past half century, many statistical studies have been devoted to paired comparisons from various perspectives, including but not limited to, ties in paired comparison experiments \cite{TPCETM,TPCEBT}, dynamic Bradley-Terry models concerning changeable merits \cite{DBTM,DSM,DPC,PEL}, algorithms for ranking \cite{RCDS,MMA,SPSL}, and applications in sports \cite{SMN,SGT}.

Although current models have been tremendously successful in many real-world applications such as sports and chess, they do have a number of limitations. First, it is rather complicated and inefficient to rate all individuals if there are too many of them. Second, elimination games only permit players who beat their opponents to advance to the next round and few current models take advantage of this. Third, few models provide a quantitative description to ``how much the outcome of a match is influenced by skill, or by chance'', as questioned in \cite{MCS}. Fourth, to the best knowledge of the authors, no models based on paired comparisons are directly applicable to comparisons among three or more objects. The above observations motivate questions as follows.
\begin{enumerate}[(Q-1)]
\item Can we predict results of comparisons without any rating system?
\item How to exploit eliminations to forecast the performance of individuals?
\item How to quantify the uncertainty in comparisons?
\item How to directly predict results of comparisons where $p$ individuals win out from $q$ individuals, without using paired comparisons?
\end{enumerate}

In this paper, we answer all the above questions. Following most classic models, we assume that the results of a competition depend on two families of factors: that of underlying individuals attributes called \emph{strengths} and that of interaction functions. Different from most current models, we focus on prediction without any rating system and introduce \emph{competitions} as special comparisons reminiscent of eliminations and evolutions. Moreover, we start from the structure generated by competitions and propose an original probability model called \emph{arena}. We first define the simplest family of arenas to introduce the concept and derive its basic properties, which are our answers to (Q-1) and (Q-2). Then a specific arena is generalized in the most straightforward way to fit reality. We propose an important metric of uncertainty in competitions with its consistent estimators to answer (Q-3). Actually, the arena permits comparisons among more than two individuals. The corresponding conclusions and methods can be easily generalized from paired competitions, which answers (Q-4). As Aldous stated in \cite{ERS}, ``there has been surprisingly little ``applied probability'' style mathematical treatment of the basic model.'' Hence, with some reasonable assumptions, we focus on mathematical derivations more than data analysis and simulations to pave the way for the future study.

The rest of the paper is organized as follows. In Section 2, we discuss a probability problem which is helpful to understand the concept of arenas and define arenas without fluctuations. A Bayes estimator of an individual's future results is given in Section 3, which is invariant in a sense. This is followed in Section 4 by an improvement of former arenas. An important parameter called \emph{the coefficient of fluctuations} is proposed with consistent estimators. In Section 5, we further improve arenas and illustrate the attendant influence as well as a new consistent estimator. Finally, the arena model is evaluated by data analysis and simulations. We also compare our model to classic ones on paired comparisons in Section 6. See Appendix A for the proof of a theorem in Section 5.

\section{What are arenas}

To better understand the concept of arenas, we first discuss a probability problem. Consider N=$2^{m+n}$ players are playing a game. The rules of the game are as follows.

(R1) At the beginning, every player is assigned a random number $X$ as his \emph{strength}. It does not change during this game. Assume $X$ is a continuous random variable and its probability density $p(\cdot)$ is supported on $\Theta$. Besides, each player's win/loss record is denoted by a dualistic array $(\cdot,\cdot)$ called his \emph{state}. The states of all players are $(0,0)$ at the beginning.

(R2) In the first \emph{round}, every player is randomly assigned an opponent. The player with higher strength wins this round, and his state turns into (1,0); the player with lower strength loses this round, and his state turns into (0,1).

(R3) Given $i<m$ and $j<n$, a player whose state is $(i,j)$ will be randomly arranged an opponent whose state is also $(i,j)$ in his next round. The player with higher strength wins this round, and his state turns into $(i+1,j)$; the player with lower strength loses this round, and his state turns into $(i,j+1)$.

(R4) A state $(i,j)$ is called a \emph{boundary state} if $i=m$ or $j=n$. After a player reaches a boundary state, the corresponding state is called his \emph{result} of this game and this game is over for him.

\begin{figure*}[!htb]
\centering
\subfigure[Figure 1: Elimination form] {\includegraphics[height=2in,width=2.8in,angle=0]{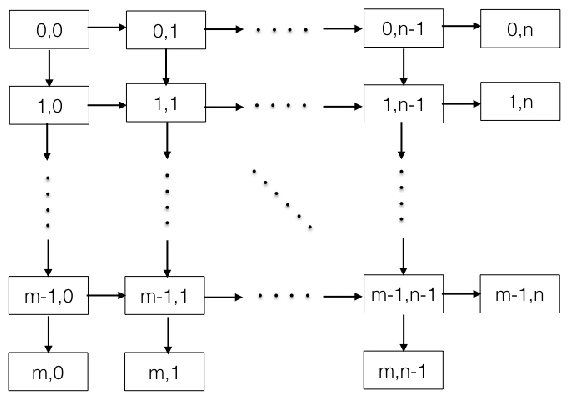}}
\subfigure[Figure 2: Bifurcation form] {\includegraphics[height=2in,width=2.8in,angle=0]{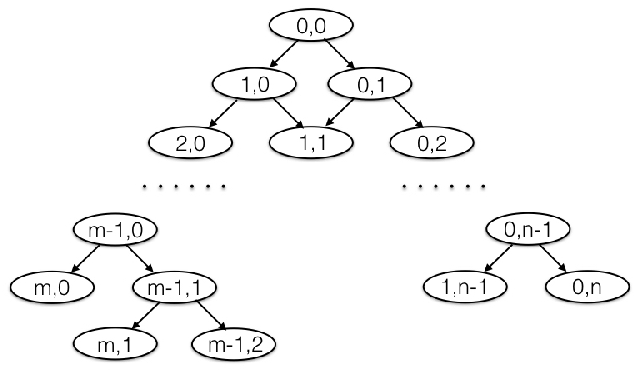}}
\end{figure*}

Figure 1 and Figure 2 describe the transformation of players' states and can help us understand the matchmaking process. We call a game that satisfies these four rules an $m$-$n$ \emph{arena game}. Since $X$ is a continuous random variable, the probability that two players' strengths are identical is zero. So we think without loss of generality that no ties happen in an arena game. First of all, we give the following theorem, which tells us the probability for a player to obtain different results.

\begin{thm}
For the above game, the possible results of a player are 
$$(m,0),(m,1),\cdots,(m,n-1),(m-1,n),(m-2,n),\cdots,(0,n).$$
Let $\mathcal{A}_{i,j}$ denote the event that the player has reached the state $(i,j)$, then we have
\begin{equation}
\begin{split}
\mathbb{P}(\mathcal{A}_{m,j})=&\binom{m+j-1}{m-1}(\frac{1}{2})^{m+j}, \qquad j = 0,1,\cdots,n-1,\\
\mathbb{P}(\mathcal{A}_{i,n})=&\binom{n+i-1}{n-1}(\frac{1}{2})^{n+i}, \qquad i = 0,1,\cdots,m-1.
\end{split}
\label{e prob no infer1}
\end{equation}
\label{t prob no infer1}
\end{thm}

\begin{proof}
For all players in state $(i,j)$, where $i<m$ and $j<n$, half of their states will turn into $(i+1,j)$ and half $(i,j+1)$. Therefore, we have
\begin{equation}
\mathbb{P}(\mathcal{A}_{i,j}) = \binom{i+j}{i}\bigl(\frac{1}{2}\bigr)^{i+j},\quad 0\leqslant i\leqslant m-1,0\leqslant j\leqslant n-1.
\end{equation}
Since there is only one source for each boundary state, we find
\begin{equation}
\mathbb{P}(\mathcal{A}_{m,j})=\frac{1}{2}\mathbb{P}(\mathcal{A}_{m-1,j})= \binom{m+j-1}{m-1}\bigl(\frac{1}{2}\bigr)^{m+j}, \quad j=0,1,\cdots,n-1,
\end{equation}
\begin{equation}
\mathbb{P}(\mathcal{A}_{i,n})=\frac{1}{2}\mathbb{P}(\mathcal{A}_{i,n-1})= \binom{n+i-1}{n-1}\bigl(\frac{1}{2}\bigr)^{n+i}, \quad i=0,1,\cdots,m-1.
\end{equation}
\end{proof}

The probability here has nothing to do with his strength and we have little information of this player. We now consider strengths. It is evident that a player in state $(i,j+1)$ is likely to have higher strength compared with a player in state $(i,j)$. The following theorem presents the probability distribution of the strength of a player in different states.

\begin{thm}
Let $X_{i,j}$ denote the strength of a player who has reached state $(i,j)$ in a run. Then $X_{i,j}$ is a continuous random variable and 
\begin{equation}
\left\{\begin{aligned}
p_{0,0}(x)&=p(x),\\
p_{i,0}(x)&=2p_{i-1,0}(x)\int_{-\infty}^xp_{i-1,0}(t)dt,\qquad(1 \leqslant i \leqslant m-1)\\
p_{0,j}(x)&=2p_{0,j-1}(x)\int_x^{+\infty}p_{0,j-1}(t)dt,\qquad(1 \leqslant j \leqslant n-1)\\
p_{i,j}(x)&=\frac{2i}{i+j}p_{i-1,j}(x)\int_{-\infty}^{x}p_{i-1,j}(t)dt+\frac{2j}{i+j}p_{i,j-1}(x)\int_x^{+\infty}p_{i,j-1}(t)dt,\\ 
&\quad(1\leqslant i\leqslant m-1,1\leqslant j\leqslant n-1)\\
p_{m,j}(x)&=2p_{m-1,j}(x)\int_{-\infty}^xp_{m-1,j}(t)dt,\qquad(0 \leqslant j \leqslant n-1)\\
p_{i,n}(x)&=2p_{i,n-1}(x)\int_x^{+\infty}p_{i,n-1}(t)dt,\qquad(0 \leqslant i \leqslant m-1)
\end{aligned}\right.
\label{e dens of str}
\end{equation}
where $p_{i,j}(\cdot)$ is the probability density of $X_{i,j}$.
\label{t dens of str}
\end{thm}

\begin{proof}
Let $X$ denote the strength of a player and $\mathcal{A}_{i,j}$ have the same meaning as we stated before, then we have $X_{0,0}\overset{\text{d}}{=}X\sim p(x)$ and
\begin{equation}
\mathbb{P}(x<X_{i,j}\leqslant x+\Delta x)=\mathbb{P}(x<X\leqslant x+\Delta x|\mathcal{A}_{i,j}).
\label{e mean of xij}
\end{equation}

(Case a)\ $1\leqslant i\leqslant m-1,j=0$ or $1\leqslant j\leqslant n-1,i=0.$ Assume $X_{i,0}$ is a continuous random variable and its probability density function is $p_{i,0}(\cdot)$. For $\Delta x>0$, according to the Theorem \ref{t prob no infer1}, (R3) and (\ref{e mean of xij}), we have
\[\begin{split}
\mathbb{P}(x<X_{i+1,0}\leqslant x+\Delta x)&=\frac{\mathbb{P}(\mathcal{A}_{i+1,0}|x< X\leqslant x+\Delta x)\mathbb{P}(x< X\leqslant x+\Delta x)}{\mathbb{P}(\mathcal{A}_{i+1,0})}\\
&=\frac{\mathbb{P}(\mathcal{A}_{i+1,0},\mathcal{A}_{i,0}|x< X\leqslant x+\Delta x)\mathbb{P}(x< X\leqslant x+\Delta x)}{\mathbb{P}(\mathcal{A}_{i+1,0})}\\
&=\frac{\mathbb{P}(\mathcal{A}_{i+1,0}|\mathcal{A}_{i,0},x< X\leqslant x+\Delta x)\mathbb{P}(x< X\leqslant x+\Delta x|\mathcal{A}_{i,0})\mathbb{P}(\mathcal{A}_{i,0})}{\mathbb{P}(\mathcal{A}_{i+1,0})}\\
&=2\mathbb{P}(x<X_{i,0}\leqslant x+\Delta x)\mathbb{P}(X_{i,0}^{(1)}<X|\mathcal{A}_{i,0},x<X\leqslant x+\Delta x),
\end{split}\]
where $X_{i,0}^{(1)}$ represents the strength of this player's opponent in the $(i+1)$-th round, whose former state is also $(i,0)$. Notice that
\begin{equation}
\begin{split}
\mathbb{P}(X_{i,0}^{(1)}<x|\mathcal{A}_{i,0},x<X\leqslant x+\Delta x)
&\leqslant\mathbb{P}(X_{i,0}^{(1)}<X|\mathcal{A}_{i,0},x<X\leqslant x+\Delta x)\\
&\leqslant\mathbb{P}(X_{i,0}^{(1)}<x+\Delta x|\mathcal{A}_{i,0},x<X\leqslant x+\Delta x).
\end{split}
\end{equation}
According to (R3), $X_{i,0}^{(1)}$ and $X$ are independent given $\mathcal{A}_{i,0}$, then we have
\begin{equation}
\mathbb{P}(x<X_{i+1,0}\leqslant x+\Delta x)=2\mathbb{P}(x<X_{i,0}\leqslant x+\Delta x)\bigl(\mathbb{P}(X_{i,0}^{(1)}<x)+O(\Delta x)\bigr)\ (\Delta x\rightarrow 0).
\end{equation}
Therefore, by induction we find $X_{i,0}$ is a continuous random variable and its density function satisfies
\begin{equation}
p_{i,0}(x)=2p_{i-1,0}(x)\int_{-\infty}^xp_{i-1,0}(t)dt, \qquad i=1,2,\cdots,m-1.
\end{equation}
Similarly, we have 
\begin{equation}
p_{0,j}(x)=2p_{0,j-1}(x)\int_{x}^{+\infty}p_{0,j-1}(t)dt, \qquad j=1,2,\cdots,n-1.
\end{equation}

(Case b)\ $1\leqslant i\leqslant m-1$ and $1\leqslant j\leqslant n-1$. For this case, notice that
\[\begin{split}
\mathbb{P}(x<X\leqslant  x+\Delta x|\mathcal{A}_{i,j})=&\frac{\mathbb{P}(\mathcal{A}_{i,j},\mathcal{A}_{i-1,j}|x<X\leqslant x+\Delta x)\mathbb{P}(x< X\leqslant x+\Delta x)}{\mathbb{P}(\mathcal{A}_{i,j})}\\
&+\frac{\mathbb{P}(\mathcal{A}_{i,j},\mathcal{A}_{i,j-1}|x<X\leqslant x+\Delta x)\mathbb{P}(x< X\leqslant x+\Delta x)}{\mathbb{P}(\mathcal{A}_{i,j})}
\end{split}\]
Then the rest of the proof is analogous to the (Case a).

(Case c)\ $i=m,0\leqslant j\leqslant n-1$ or $j=n,0\leqslant i\leqslant m-1.$ 
For players in boundary states $(m,j)$ and $(i,n)$, there is only one source: winners from state $(m-1,j)$ and losers from state $(i,n-1)$ respectively. The rest is similar to the first case.
\end{proof}
Deriving distribution functions from density functions gives the following corollary.
\begin{coro}
Let $F_{i,j}(x)$ denote the distribution function the strength of a player who has reached state $(i,j)$ in an arena game. Then we have
\begin{equation}
\left\{
\begin{aligned}
F_{0,0}(x)&=\int_{-\infty}^{x}p(t)\dif t,\\
F_{i,0}(x)&=\bigl(F_{i-1,0}(x)\bigr)^2,\qquad(0 \leqslant i \leqslant m-1)\\
F_{0,j}(x)&=1-\bigl(1-F_{0,j-1}(x)\bigr)^2,\qquad(0 \leqslant j \leqslant n-1)\\
F_{i,j}(x)&=\frac{i}{i+j}\bigl(F_{i-1,j}(x)\bigr)^2+\frac{j}{i+j}\bigl(1-\bigl(1-F_{i,j-1}(x)\bigr)^2\bigr),\\ 
&\quad(1 \leqslant i \leqslant m-1,1 \leqslant j \leqslant n-1)\\
F_{m,j}(x)&=\bigl(F_{m-1,j}(x)\bigr)^2,\qquad(0 \leqslant j \leqslant n-1)\\
F_{i,n}(x)&=1-\bigl(1-F_{i-1,n}(x)\bigr)^2.\qquad(0 \leqslant i \leqslant m-1)
\end{aligned}\right.
\label{e dist of str}
\end{equation}
\label{t dist of str}
\end{coro}

In fact, we can derive the distribution of a player's strength in different states from order statistics. According to (R2), the first round of a run can be viewed as simple random sampling from the original distribution. The states of players with higher strengths turn into (1,0) while the ones of players with lower strengths turn into (0,1). Therefore, we have
\begin{equation}
X_{1,0}\overset{d}{=}max\{X_{0,0}^{(1)},X_{0,0}^{(2)}\},\qquad X_{0,1}\overset{d}{=}min\{X_{0,0}^{(1)},X_{0,0}^{(2)}\}
\end{equation}
respectively, where $X_{0,0}^{(1)},X_{0,0}^{(2)}\ i.i.d \sim p_{0,0}(\cdot)$. Similarly, we have
\begin{equation}
X_{i,0}\overset{d}{=}max\{X_{i-1,0}^{(1)},X_{i-1,0}^{(2)}\},\qquad X_{0,j}\overset{d}{=}min\{X_{0,j-1}^{(1)},X_{0,j-1}^{(2)}\},
\end{equation}
where
\begin{equation}
X_{i-1,0}^{(1)},X_{i-1,0}^{(2)}\ i.i.d \sim p_{i-1,0}(\cdot)\ ,\ X_{0,j-1}^{(1)},X_{0,j-1}^{(2)}\ i.i.d \sim p_{0,j-1}(\cdot).
\end{equation}
For players in state $(i,j)(1\leqslant i\leqslant m-1,1\leqslant j\leqslant n-1)$, the population includes two groups of players that winners from state $(i-1,j)$ and losers from state $(i,j-1)$, whose ratio can be uniquely determined by Theorem \ref{t prob no infer1}. Through this, some results can be easily generalized for competitions among three or more individuals. We do not discuss such case in this paper and leave it for some further research.

Theorem \ref{t prob no infer1} tells us how many people are there in different states. However, what interests us is the probability for a player, whose strength is already known as $x$, to obtain different results. The following theorem answers this.

\begin{thm}
In an $m$-$n$ arena game, suppose X is the random number assigned to a player with the probability density $p_{0,0}(\cdot)$ and x is a real number in the $\Theta$. Let $\mathcal{A}_{i,j}$ and $p_{i,j}(\cdot)$ have the same meaning as the above, then
\begin{equation}
\begin{split}
\mathbb{P}(\mathcal{A}_{m,j}|X=x)&=\mathbb{P}(\mathcal{A}_{m,j})\cdot \frac{p_{m,j}(x)}{p_{0,0}(x)},  \qquad j=0,1,\cdots,n-1,\\
\mathbb{P}(\mathcal{A}_{i,n}|X=x)&=\mathbb{P}(\mathcal{A}_{i,n})\cdot \frac{p_{i,n}(x) }{p_{0,0}(x)}, \qquad i=0,1,\cdots,m-1.
\end{split}
\label{e prob with infer}
\end{equation}
\label{t prob with infer}
\end{thm}

\begin{proof}
The definition of $\mathcal{A}_{i,j}$, Theorem \ref{t prob no infer1} and \ref{t dens of str} yield
\begin{equation}
\begin{split}
\mathbb{P}(\mathcal{A}_{m,j}|X=x)&=\lim\limits_{\Delta x \rightarrow 0}\frac{\mathbb{P}(\mathcal{A}_{m,j},x<X\leqslant x+\Delta x)}{\mathbb{P}(x<X\leqslant x+\Delta x)}\\
&=\lim\limits_{\Delta x \rightarrow 0} \frac{\mathbb{P}(x<X_{m,j}\leqslant x+\Delta x)\mathbb{P}(\mathcal{A}_{m,j})}{p_{0,0}(x)\Delta x}\\
&=\mathbb{P}(\mathcal{A}_{m,j})\cdot \frac{p_{m,j}(x)}{p_{0,0}(x)},\qquad j=0,1,\cdots ,n-1.
\end{split}
\end{equation}
By the same token, we have
\begin{equation}
\mathbb{P}(\mathcal{A}_{i,n}|X=x)=\mathbb{P}(\mathcal{A}_{i,n})\cdot \frac{p_{i,n}(x) }{p_{0,0}(x) },\qquad i=0,1,\cdots,m-1.
\end{equation}
\end{proof}

Further, we can define a random variable to indicate the probabilities for a player to obtain different results in an $m$-$n$ arena game, given his strength.

\begin{defn}
Suppose $m,n\in \mathbb{N},\lambda\in(0,1)$ and a two-dimensional vector $\xi$ takes values on 
$$\bigl\{(m,j):j=0,1,\cdots,n-1\bigr\} \cup \bigl\{(i,n):i=0,1,\cdots,m-1\bigr\},$$ 
and its probability mass function is given by
\begin{equation}
\mathbb{P}\bigl(\xi=(m,j)\bigr)=\bigl(\frac{1}{2}\bigr)^{m+j}\binom{m+j-1}{m-1}\frac{p_{m,j}(\lambda)}{p_{0,0}(\lambda)},\qquad j=0,1,\cdots,n-1,
\end{equation}
\begin{equation}
\mathbb{P}\bigl(\xi=(i,n)\bigr)=\bigl(\frac{1}{2}\bigr)^{n+i}\binom{n+i-1}{n-1}\frac{p_{i,n}(\lambda)}{p_{0,0}(\lambda)},\qquad i=0,1,\cdots,m-1,
\end{equation}
where the $p_{i,j}(\cdot)$ follows Theorem \ref{t dens of str}. Then we call $\xi$ an \emph{arena random variable} with parameter $\lambda$, $m$, $n$ and $p(\cdot)$, which are called \emph{strength}, \emph{win threshold}, \emph{loss threshold} and \emph{original density} respectively. Simply denote it as $\xi\sim Arena(\lambda;m,n,p)$.
\label{d dist arena}
\end{defn}

It should be emphasized that $\mathbb{P}\bigl(\xi=(m,j)\bigr)$ and $\mathbb{P}\bigl(\xi=(i,n)\bigr)$ represent the probabilities that a player with strength $\lambda$ obtains different results, while $p_{m,j}$ and $p_{i,n}$ the probabilities without any information.

\section{Arena model without fluctuations}

We have got a complete solution to this probability problem, which includes the percentage of players with different results, probability distributions of strengths of players in different states, and the probability for a player to obtain different results with a fixed strength. These results will serve as the basis of the following inferences. Actually, a player's strength always changes little in a relatively short time. As a result, we can give a Bayesian inference on the strength and results of a player, under the assumption that every player has a constant as his strength. We now establish the arena without fluctuations and do that in this Section.

\subsection{Assumptions}

In this part, four assumptions are given, according to which an $m$-$n$ \emph{arena without fluctuations with original density} $p(\cdot)$ can be established.

(A1) In an arena, an infinite number of \emph{runs} can be held among a fixed group of individuals. These individuals are called \emph{players}. All players constitute a countably infinite set $A_{0,0}^{q}=\{a_1,a_2,\cdots\}$, where $a_l$ is the $l$-th player and $q=1,2,\cdots$.

(A2) Each player has an observable \emph{state} $(i,j)\in\varepsilon$ with respect of time and an unobservable constant \emph{strength} $x\in\mathbb{R}$, where
$$\varepsilon=\{(i,j):0\leqslant i\leqslant m,0\leqslant j\leqslant n\}\backslash\{(m,n)\}.$$
Denote the strength of the $l$-th player by $X_l$. Assume $X_1,X_2,\cdots\,X_n,\cdots$ are independent and identically distributed, supported on $\Theta$, and their density function is $p(x)$.

(A3) Let $A_{i,j}^{q}$ denote the set of players whose states are $(i,j)$ after $(i+j)$-th round in the $q$-th run. If
$$0\leqslant i\leqslant m-1,0\leqslant j\leqslant n-1,$$
then the system will \emph{randomly assign} him an opponent $a_{l'}$ from $A_{i,j}^{q}$. If $X_{l}>X_{l'}$, then let
$$a_l\in A_{i+1,j}^{q}, a_{l'}\in A_{i,j+1}^{q}.$$
Otherwise, let
$$a_l\in A_{i,j+1}^{q}, a_{l'}\in A_{i+1,j}^{q}.$$

(A4) If a player's state satisfies $i=m$ or $j=n$ in the $q$-th run, then we say the player's $q$-th run is over and the $(i,j)$ is called his \emph{result} of the $q$-th run. When all players' $q$-th runs are over, a new run will start according to (A3). At the same time, their numbers of runs $q$ plus one.

In a word, an $m$-$n$ arena without fluctuations consists of infinitely many $m$-$n$ arena games, where players' constant strengths will not be affected by their previous results. As we stated about arena games before, no ties happen in an arena without fluctuations as well, due to Assumption (A2). In (A3), we mentioned that the assignment is random. What should a random assignment be? We all know that assignments in a round are idempotent; that is, if player A's opponent is player B, then player B's opponent must be player A. Besides, the term random here is not the one in the strict sense since a uniform random variable cannot be supported on a countably infinite set. Hence, we define the matching system here.

\begin{defn}
$\sigma:A\rightarrow A$ is called a matching map on a nonempty set A if
$$\forall a\in A,\sigma(a)\neq a,\sigma\bigl(\sigma(a)\bigr)=a.$$
\end{defn}
\begin{defn}
$\sigma:\Omega\times A_M\rightarrow A_M$ is called a random matching on a finite set $A_M$ ($|A_M|=M$ is even) in the probability space $(\Omega,\mathcal{F},\mathbb{P})$, if $\forall \omega\in\Omega$, $\sigma(\omega,\cdot)$ is a matching map on $A_M$ and
\begin{equation}
\forall a\in A_M,\forall a'\in A_M\backslash\{a\},\mathbb{P}\bigl(\sigma(a)=a'\bigr)=\frac{1}{M-1}.
\end{equation}
\label{d matc fty}
\end{defn}

In other words, a random matching can be viewed as a stochastic process in the probability space $(\Omega,\mathcal{F},\mathbb{P})$. Both the index set and state space are $A_M$. Now we regard the random matching on a countably infinite set $A$ as a pro forma limit of the one on finite set $A_M$, which means $\mathbb{P}\bigl(\sigma(a)=a'\bigr)$ could be infinitely small. Hence, the random matching on a countably infinite set A is supposed to satisfy $\mathbb{P}(\sigma(a)=a')=0$ though strictly speaking $\sigma$ is no longer a stochastic process.

\begin{defn}
Assume $A$ is a countably infinite set and $\sigma$ is called a random matching on A in the probability space $(\Omega,\mathcal{F},\mathbb{P})$ if
\begin{equation}
\forall a,a'\in A,\mathbb{P}\bigl(\sigma(a)=a'\bigr)=0\ and \ \mathbb{P}\bigl(\sigma\bigl(\sigma(a)\bigr)=a\bigr)=1.
\end{equation}
\label{d matc infty}
\end{defn}

We have assumed that every player has an unobservable constant as his strength. According to Assumption (A1), (A3), (A4) and Definition \ref{d matc infty}, the opponents of a player in all rounds are different with probability one, and their strengths are merely related to the state of the player. A Markov chain emerges if we consider the transition of the states of a player. Before proving it, we define a specific Markov chain.

\begin{defn}
Let $\mathbf{S}=(S_{t},t\geqslant1)$ be a Markov chain whose state space is
$$\varepsilon=\bigl\{(i,j):0\leqslant i\leqslant m,0\leqslant j\leqslant n\bigr\}\backslash\bigl\{(m,n)\bigr\},$$
where m,n are positive integers and its probability transfer function satisfies:\par
if $i=m$ or $j=n$, then
\begin{equation}
\mathbb{P}\bigl(S_{k+1}=(0,0)|S_{k}=(i,j)\bigr)=1,
\label{e proc arena bound}
\end{equation}

if $0\leqslant i\leqslant m-1$ or $0\leqslant j\leqslant n-1$, then
\begin{equation}
\begin{split}
\mathbb{P}\bigl(S_{k+1}=(i+1,j)|S_{k}=(i,j)\bigr)&=F_{i,j}(\lambda),\\
\mathbb{P}\bigl(S_{k+1}=(i,j+1)|S_{k}=(i,j)\bigr)&=1-F_{i,j}(\lambda),
\end{split}
\label{e proc arena not bound}
\end{equation}
where $F_{i,j}(\cdot)$ follows (\ref{e dist of str}). Then we call $\mathbf{S}$ an $m$-$n$ arena process without fluctuations with parameter $\lambda$ and original distribution $F(\cdot)$, where m,n are called win threshold and loss threshold respectively. Simply write
$$\mathbf{S}\sim A(\lambda;m,n,F).$$
\label{d proc arena}
\end{defn}

\subsection{Bayesian inference on results in arenas without fluctuations}

One of our goals is to infer a player's future performance from his past performance. Its significance is supported by many practical examples, such as sports game like FIFA World Cup and population competitions in Biology. In this part, we only consider the constant strengths of players and exclude unstable factors, although ``chance'' does contribute the performance of players. The improve work of that will be done in the next Section.

\begin{thm}
In an m-n arena without fluctuations with original density $p(\cdot)$, a player's states are recorded by
$$\mathbf{S}=\{S_t=(i_t,j_t),t\geqslant 1\}.$$
Suppose $X$ is the strength of this player, $\lambda\in\Theta$, and $F(x)=\int_{-\infty}^{x}p(t)\dif t$, then we have
$$\mathbf{S}|X=\lambda\sim A(\lambda;m,n,F).$$
\label{t proc arena}
\end{thm}

\begin{proof}
Let $F_{i,j}(x)$ have the same meaning in Corollary \ref{t dist of str}. For given $(i_1,j_1),(i_2,j_2),\cdots,(i_k,j_k)$ satisfying
\begin{equation}
\begin{cases}
i_t=j_t=0, &i_{t-1}=m\ or\ j_{t-1}=n,\\
i_t-i_{t-1}=1\ or\ j_t-j_{t-1}=1,&otherwise,
\end{cases}
\end{equation}
if $i_k=m$ or $j_k=n$, then according to Assumption (A4),
$$\mathbb{P}\bigl(S_{k+1}=(0,0)|S_k=(i_k,j_k),X=\lambda\bigr)=\mathbb{P}\bigl(S_{k+1}=(0,0)|S_k=(i_k,j_k),\cdots,S_1=(i_1,j_1),X=\lambda\bigr)=1.$$
If $0\leqslant i_k\leqslant m-1$ or $0\leqslant j_k\leqslant n-1$, define $T=\{0\leqslant t\leqslant k-1:0\leqslant i_t\leqslant m-1,0\leqslant j_t\leqslant n-1\}$. For $t\in T\cup\{k\}$, let
$$\sigma_{i_t,j_t}^{q_t}:\Omega\times A_{i_t,j_t}^{q_t}\rightarrow A_{i_t,j_t}^{q_t},$$
be a random matching on $A_{i_t,j_t}^{q_t}$. According to Assumptions (A1), (A3) and Definition \ref{d matc infty},
\begin{equation}
\mathbb{P}\left(\forall t\in T,\sigma_{i_t,j_t}^{q_t}\neq\sigma_{i_k,j_k}^{q_k}\middle|S_k=(i_k,j_k),\cdots,S_1=(i_1,j_1),X=\lambda\right)=1.
\label{e match condition}
\end{equation}
Assumption (A2) and (\ref{e match condition}) yield
\[\begin{split}
&\mathbb{P}\bigl(S_{k+1}=(i_k+1,j_k)|S_k=(i_k,j_k),\cdots,S_1=(i_1,j_1),X=\lambda\bigr)\\
=&\mathbb{P}\left(X>X_{\sigma_{i_k,j_k}^{q_k}}\middle|S_k=(i_k,j_k),\cdots,S_1=(i_1,j_1),X=\lambda\right)\\
=&\mathbb{P}\left(X>X_{\sigma_{i_k,j_k}^{q_k}},\forall t\in T,\sigma_{i_t,j_t}^{q_t}\neq\sigma_{i_k,j_k}^{q_k}\middle|S_k=(i_k,j_k),\cdots,S_1=(i_1,j_1),X=\lambda\right)\\
=&\mathbb{P}\left(X_{\sigma_{i_k,j_k}^{q_k}}<\lambda\right).
\end{split}\]
Corollary \ref{t dist of str} gives the distributions of strengths of players in different states in an $m$-$n$ arena game. According to Assumption (A1) and (A4), the distributions of the strengths of players in an $m$-$n$ arena is the same as that result. It follows that
\begin{equation}
\begin{split}
&\mathbb{P}\bigl(S_{k+1}=(i_k+1,j_k)|S_k=(i_k,j_k),\cdots,S_1=(i_1,j_1),X=\lambda\bigr)\\
=&\mathbb{P}\bigl(S_{k+1}=(i_k+1,j_k)|S_k=(i_k,j_k),X=\lambda\bigr)=F_{i_k,j_k}(\lambda),
\end{split}
\end{equation}
where $F_{i,j}$ follows (\label{e dist of str}). Similarly,
\begin{equation}
\begin{split}
&\mathbb{P}\bigl(S_{t+1}=(i_t,j_t+1)|S_k=(i_t,j_t),\cdots,S_1=(i_1,j_1),X=\lambda\bigr)\\
=&\mathbb{P}\bigl(S_{t+1}=(i_t,j_t+1)|S_t=(i_t,j_t),X=\lambda\bigr)=1-F_{i_t,j_t}(\lambda).
\end{split}
\end{equation}
In conclusion, by Definition \ref{d proc arena} we have 
$$\mathbf{S}|X=\lambda\sim A(\lambda;m,n,F).$$
\end{proof}

In Assumption (A1), the original density characterizes the strengths of all players. Hence, the original density is an appropriate choice of our prior distribution. Furthermore, by choosing the original density as the prior distribution, we can derive an invariant estimator which is independent of it. We first give the posterior distribution of a player's strength.

\begin{thm}
Suppose a player plays $k$ rounds in an $m$-$n$ arena without fluctuations with original density $p(\cdot)$ and his states turn successively into
$$\tilde{x}=\bigl((i_{1},j_{1}),(i_{2},j_{2}),\cdots,(i_{k},j_{k})\bigr).$$
Choose $p(\cdot)$ to be the prior distribution of his strength X, then the posterior distribution of X is 
\begin{equation}
\pi(\lambda|\tilde{x})=\frac{\prod\limits_{t\in I_1}\Bigl[(i_{t}-i_{t-1})F_{i_{t-1},j_{t-1}}(\lambda)+(j_{t}-j_{t-1})\bigl(1-F_{i_{t-1},j_{t-1}}(\lambda)\bigr)\Bigr]p(\lambda)}{\int_\Theta \prod\limits_{t\in I_1}\Bigl[(i_{t}-i_{t-1})F_{i_{t-1},j_{t-1}}(\lambda)+(j_{t}-j_{t-1})\bigl(1-F_{i_{t-1},j_{t-1}}(\lambda)\bigr)\Bigr]p(\lambda) \dif \lambda},
\label{e post dest}
\end{equation}
where $F_{i,j}$ follows (\ref{e dist of str}) and $I_1=\{1\leqslant t \leqslant k:i_{t}>i_{t-1}\ or\ j_{t}>j_{t-1}\}.$
\label{t post dest}
\end{thm}

\begin{proof}
Suppose $S_t$ is the state of this player after his $(t-1)$-th round and $\mathbf{S}=(S_t,t\geqslant 1)$. Based on Theorem \ref{t proc arena}, we have
$$\mathbf{S}|X=\lambda\sim A(\lambda;m,n,F).$$
As a result, the posterior distribution of $X$ is 
\[\begin{split}
\pi(\lambda|\tilde{x})&=\frac{p(\tilde{x}|\lambda)\pi(\lambda)}{\int_\Theta p(\tilde{x}|\lambda)\pi(\lambda) \dif \lambda}\\
&=\frac{\prod \limits_{t=1}^{k}\mathbb{P}\left(S_{t}=(i_{t},j_{t})\middle|S_{t-1}=(i_{t-1},j_{t-1}),X=\lambda)p(\lambda\right)}{\int_{\Theta}\prod \limits_{t=1}^{k}\mathbb{P}\left(S_{t}=(i_{t},j_{t})\middle|S_{t-1}=(i_{t-1},j_{t-1}),X=\lambda\right)p(\lambda)\dif \lambda}\\
&=\frac{\prod\limits_{t\in I_1}\Bigl[(i_{t}-i_{t-1})F_{i_{t-1},j_{t-1}}(\lambda)+(j_{t}-j_{t-1})\bigl(1-F_{i_{t-1},j_{t-1}}(\lambda)\bigr)\Bigr]p(\lambda)}{\int_\Theta \prod\limits_{t\in I_1}\Bigl[(i_{t}-i_{t-1})F_{i_{t-1},j_{t-1}}(\lambda)+(j_{t}-j_{t-1})\bigl(1-F_{i_{t-1},j_{t-1}}(\lambda)\bigr)\Bigr]p(\lambda) \dif \lambda}.
\end{split}\]
\end{proof}

The word round here is different from the one in our assumptions. On the one hand, the round in our assumptions is a finite value in a run. The round here represents a value that can be arbitrarily large and does not return to zero even though a run ends. On the other hand, state (0,0) is deliberately inserted before the start of a run.

The next theorem shows a significant property of an arena without fluctuations. The posterior predictive distribution usually relies on the prior distribution chosen by us. Nevertheless, the posterior distribution of a player's future results in an arena without fluctuations is irrelevant to the prior distribution of player's strength $p(\cdot)$.

\begin{thm}
Choose $p(\cdot)$ in Assumption (A1) to be the prior distribution of a player's strength $X$. Given $k$ successive states of a player
$$\tilde{x}=(x_{1},x_{2},\cdots,x_{k})=\bigl((i_{1},j_{1}),(i_{2},j_{2}),\cdots,(i_{k},j_{k})\bigr),$$
then the next $r$ states of this player $(x_{k+1},x_{k+2},\cdots,x_{k+r})$ has a posterior predictive distribution irrelevant to $p(\cdot)$.
\label{t post dist invari}
\end{thm}

\begin{proof}
Denote the states sequence of this player by $\mathbf{S}=(S_t,t\geqslant 1).$ Theorem \ref{t proc arena} gives
$$\mathbf{S}|X=\lambda\sim A(\lambda;m,n,F).$$
Therefore, we have 
\begin{equation}
\begin{split}
&\mathbb{P}(S_{t+1}=(i_{k+1},j_{k+1}),\cdots,S_{k+r}=(i_{k+r},j_{k+r})|\tilde{x})\\
=&\int_{\Theta}\mathbb{P}(S_{k+1}=(i_{k+1},j_{k+1}),\cdots,S_{k+r}=(i_{k+r},j_{k+r})|\tilde{x},X=\lambda)\pi(\lambda|\tilde{x})\dif \lambda\\
=&\int_{\Theta}\prod\limits_{t=k}^{k+r-1}\mathbb{P}(S_{t+1}=(i_{t+1},j_{t+1})|S_t=(i_t,j_t),X=\lambda)\pi(\lambda|\tilde{x})\dif \lambda,
\end{split}
\label{e post prob1}
\end{equation}
Notice that if $i_{t}=m$ or $j_{t}=n$,
\begin{equation}
\mathbb{P}\bigl(S_{t+1}=(i_{t+1},j_{t+1})|S_t=(i_{t},j_{t})\bigr)=\begin{cases}
1,&i_{t+1}=j_{t+1}=0,\\
0,&otherwise.
\end{cases}
\label{e cond prob of proc1}
\end{equation}
If $0\leqslant i_t\leqslant m-1$ and $0\leqslant j_t\leqslant n-1$,
\begin{equation}
\mathbb{P}\bigl(S_{t+1}=(i_{t+1},j_{t+1})|S_t=(i_{t},j_{t})\bigr)=
\begin{cases}
F_{i_{t},j_{t}}(\lambda),&i_{t+1}=i_t+1\ and \ j_{t+1}=j_t,\\
1-F_{i_{t},j_{t}}(\lambda),&i_{t+1}=i_t\ and \ j_{t+1}=j_t+1,\\
0,&otherwise.
\end{cases}
\label{e cond prob of proc2}
\end{equation}
According to Corollary \ref{t dist of str}, it is easy to prove by induction that function $F_{i,j}(\cdot)$ can be expressed in the form
\begin{equation}
F_{i,j}(\lambda)=g_{i,j}\bigl(F(\lambda)\bigr),\qquad (0\leqslant i\leqslant m,0\leqslant j\leqslant n)
\label{e dist poly}
\end{equation}
where $g_{i,j}(\cdot)$ is a polynomial function. Substituting (\ref{e post dest}), (\ref{e cond prob of proc1}), (\ref{e cond prob of proc2}), and (\ref{e dist poly}) into (\ref{e post prob1}) yields
\[
\begin{split}
&\mathbb{P}\left(S_{t+1}=(i_{k+1},j_{k+1}),\cdots,S_{k+r}=(i_{k+r},j_{k+r})\middle|\tilde{x}\right)\\
=&\frac{\int_{\Theta}\prod\limits_{t\in I_2}\Bigl\{(i_{t}-i_{t-1})g_{i_{t-1},j_{t-1}}\bigl(F(\lambda)\bigr)+(j_{t}-j_{t-1})\Bigl[1-g_{i_{t-1},j_{t-1}}\bigl(F(\lambda)\bigr)\Bigr]\Bigr\}p(\lambda)\dif  \lambda}{\int_{\Theta}\prod\limits_{t\in I_1}\Bigl\{(i_{t}-i_{t-1})g_{i_{t-1},j_{t-1}}\bigl(F(\lambda)\bigr)+(j_{t}-j_{t-1})\Bigl[1-g_{i_{t-1},j_{t-1}}\bigl(F(\lambda)\bigr)\Bigr]\Bigr\}p(\lambda)\dif  \lambda}\\
=&\frac{\int_{0}^{1}\prod\limits_{t\in I_2}\Bigl[(i_{t}-i_{t-1})g_{i_{t-1},j_{t-1}}(u)+(j_{t}-j_{t-1})\bigl(1-g_{i_{t-1},j_{t-1}}(u)\bigr)\Bigr]\dif  u}{\int_{0}^{1}\prod\limits_{t\in I_1}\Bigl[(i_{t}-i_{t-1})g_{i_{t-1},j_{t-1}}(u)+(j_{t}-j_{t-1})\bigl(1-g_{i_{t-1},j_{t-1}}(u)\bigr)\Bigr]\dif  u},
\end{split}
\]
where
$$I_1=\{1\leqslant t \leqslant k:i_{t}>i_{t-1}\ or\ j_{t}>j_{t-1}\},$$
$$I_2=\{1\leqslant t \leqslant k+r:i_{t}>i_{t-1}\ or\ j_{t}>j_{t-1}\}.$$
Since $g_{i_{t},j_{t}}(\cdot)$ is irrelevant to $u=F(\lambda)$, $p_r(x_{k+1},\cdots,x_{k+r}|\tilde{x})$ can be computed by the integral of polynomials, which completes the proof.
\end{proof}

\section{Arena model with fluctuations in the infinite case}

The previous arena strictly follows the jungle law---one who has higher strength wins the game. Nevertheless, accidents may happen in the real world. For example, it is possible that a green hand wins a veteran due to his good luck. Hence, an improved model considering fluctuations should be established. Four assumptions are set here to establish an $m$-$n$ \emph{arena with fluctuations}:

(A1) In an arena, an infinite number of \emph{runs} can be held among a fixed group of individuals, and these individuals are called \emph{players}. All players constitute a countably infinite set $A_{0,0}^{q}=\{a_1,a_2,\cdots\}$, where $a_l$ is the $l$-th player and $q=1,2,\cdots$.

(A2') For each player, there is an observable $(i,j)\in\varepsilon$ as his \emph{state} with respect of time and an unobservable constant $x\in\mathbb{R}$ as his \emph{strength}, where
$$\varepsilon=\{(i,j):0\leqslant i\leqslant m,0\leqslant j\leqslant n\}\backslash\{(m,n)\}.$$
Denote the strength of the $l$-th player as $X_l$. Assume $X_1,X_2,\cdots\,X_n,\cdots $ are independent and identically distributed, and their density function is $p(x)$. Call
\begin{equation}
X_l^{q,k}=X_l+\frac{\rho_l}{\sqrt{2}}\epsilon_l^{q,k}
\end{equation}
the performance of the $l$-th player in the $k$-th round of his $q$-th run, where $\rho_l>0$ is an unknown value called the coefficient of fluctuations of the $l$-th player and $\epsilon_{l}^{q,k}$ is the relative fluctuations of the $l$-th player in the $k$-th round of the $q$-th run. Assume 
$$\epsilon_l^{1,1},\epsilon_l^{1,2},\cdots,\epsilon_l^{2,1},\epsilon_l^{2,2},\cdots,\epsilon_l^{3,1},\epsilon_l^{3,2},\cdots i.i.d \sim N(0,1)$$
and $X_l,\epsilon_l^{q,k},\epsilon_{l'}^{q',k'}$ are mutually independent for arbitrary $q,q', k,k'$ and $l\neq l'$.

(A3') Let $A_{i,j}^{q}$ denote the set of players whose states are $(i,j)$ after $(i+j)$-th round in the $q$-th run. If
$$0\leqslant i\leqslant m-1,0\leqslant j\leqslant n-1,$$
then the system will \emph{randomly assign} an opponent $a_{l'}$ from $A_{i,j}^{q}$ to him. If $X_{l}^{q,i+j+1}>X_{l'}^{q,i+j+1}$, then let
$$a_l\in A_{i+1,j}^{q}, a_{l'}\in A_{i,j+1}^{q}.$$
Otherwise, let
$$a_l\in A_{i,j+1}^{q}, a_{l'}\in A_{i+1,j}^{q}.$$

(A4') If a player's state satisfies $i=m$ or $j=n$ in the $q$-th run, then we say the player's $q$-th run is over and this state $(i,j)$ is called his \emph{result} of the $q$-th run. When all players' $q$-th runs are over, a new run will start according to (A3'). At the same time, their numbers of runs $q$ plus one.

We also think there are no ties due to Assumption (A2'). Specifically, an arena with fluctuations is called an \emph{arena with uniform fluctuations} if
\begin{equation}
\forall l\in \mathbb{N},\ \rho_l=\rho>0.
\end{equation}
In this case, $\rho$ is called the coefficient of fluctuations. This coefficient depicts the variance of fluctuations and reflects the ``fairness'' of a competition. Here ``fairness'' refers to a great probability that a player with high strength wins the one with low strength; in other words, every player deserves for his results according to his strength. It is easy to understand this concept by comparing two games that chess and finger-guessing. For simplicity, only 1-1 arena with uniform fluctuations is studied in this paper, where a player's result of a run is either win (his result is (1,0) and simply write 1) or loss (his result is (0,1) and simply write 0).

In a 1-1 arena, a player is likely to meet any other players because they all compete in the same state (0,0). In this part, we discuss the estimation of the coefficient of fluctuations based on our four basic assumptions. We call this part infinite case because the population of players is a countably infinite set according to Assumption (A1). In this case, we can express the result of the $l$-th player in his $k$-th round in the form
\begin{equation}
I_{lk}=\mathbf{1}\left\{X_l+\frac{\rho}{\sqrt{2}}\epsilon_{l,k}>X_{\sigma_k(l)}+\frac{\rho}{\sqrt{2}}\epsilon_{\sigma_k(l),k}\right\},
\label{e defn Ilk}
\end{equation}
where $\sigma_k$ is a random matching on $A_{0,0}^{q}$ and $\epsilon_{l,k}$ is the relative fluctuations of the $l$-th player in his $k$-th round. We can see that the $l$-th player wins his $k$-th round iff $I_{lk}=1.$
The following theorem gives the conditional distribution of $I_{lk}$ in a 1-1 arena with uniform fluctuations.

\begin{thm}
For arbitrary $l\in \mathbb{N}$ and $\lambda\in\mathbb{R}$,
$$I_{l1}|X_l=\lambda,I_{l2}|X_l=\lambda,\cdots,I_{ln}|X_l=\lambda\ i.i.d\sim Bernoulli\left(\Phi\bigl(\frac{\lambda}{\sqrt{1+\rho^2}}\bigr)\right),$$
where $X_l$ is the strength of the $l$-th player and $I_{lk}=\mathbf{1}$\{The $l$-th player wins his $k$-th round\}.
\label{t cond dist}
\end{thm}

\begin{proof}
By Assumption (A2') we straightforwardly have 
\[\begin{split}
\mathbb{P}(I_{lk}=1|X_l=\lambda)&=\mathbb{P}\left(X_l+\frac{\rho}{\sqrt{2}}\epsilon_{l,k}>X_{\sigma_k(l)}+\frac{\rho}{\sqrt{2}}\epsilon_{\sigma_k(l),k}\middle|X_l=\lambda\right)\\
&=\mathbb{P}\left(X_{\sigma_k(l)}+\frac{\rho}{\sqrt{2}}\epsilon_{\sigma_k(l),k}-\frac{\rho}{\sqrt{2}}\epsilon_{l,k}<\lambda\middle|X_l=\lambda\right).
\end{split}\]
On the ground of Definition \ref{d matc infty},
$$X_l,X_{\sigma_k(l)},\epsilon_{\sigma_k(l),k},\epsilon_{l,k}\ i.i.d\sim N(0,1),$$
and therefore,
\begin{equation}
\mathbb{P}(I_{lk}=1|X_l=\lambda)=\mathbb{P}(X_{\sigma_k(l)}+\frac{\rho}{\sqrt{2}}\epsilon_{\sigma_k(l),k}-\frac{\rho}{\sqrt{2}}\epsilon_{l,k}<\lambda)=\Phi\bigl(\frac{\lambda}{\sqrt{1+\rho^2}}\bigr).
\label{e cond prob of win1}
\end{equation}
Moreover,
$$X_{\sigma_1(l)},\epsilon_{\sigma_1(l),1},\epsilon_{l,1},X_{\sigma_2(l)},\epsilon_{\sigma_2(l),2},\epsilon_{l,2},\cdots,X_{\sigma_n(l)},\epsilon_{\sigma_n(l),n},\epsilon_{l,n}\ i.i.d\sim N(0,1).$$
Combining it with (\ref{e cond prob of win1}) gives
$$I_{l1}|X_l=\lambda,I_{l2}|X=\lambda,\cdots,I_{ln}|X=\lambda\ i.i.d\sim Bernoulli\left(\Phi\bigl(\frac{\lambda}{\sqrt{1+\rho^2}}\bigr)\right).$$
\end{proof}
With the results of $m$ players in their $n$ rounds, we can write a matrix
\[I=\left(\begin{array}{cccc}
I_{11}&I_{12}&\cdots&I_{1n}\\
I_{21}&I_{22}&\cdots&I_{2n}\\
\vdots&\vdots&\ddots&\vdots\\
I_{m1}&I_{m2}&\cdots&I_{mn}\\
\end{array}\right),\]
where $I_{lk}=\mathbf{1}$\{The $l$-th player wins his $k$-th round\}.
According to Theorem \ref{t cond dist}, we have
\begin{equation}
\mathbb{E}\bigl(\frac{1}{n}\sum_{k=1}^{n}I_{lk}\bigr)=\mathbb{E}\Phi\bigl(\frac{X_l}{\sqrt{1+\rho^2}}\bigr),
\end{equation}
\begin{equation}
\mathbb{E}\bigl(\frac{1}{n}\sum_{k=1}^{n}I_{lk}\bigr)^2=\frac{1}{n}\mathbb{E}\Phi\bigl(\frac{X_l}{\sqrt{1+\rho^2}}\bigr)+\bigl(1-\frac{1}{n}\bigr)\mathbb{E}\Phi^2\bigl(\frac{X_l}{\sqrt{1+\rho^2}}\bigr),
\end{equation}
where $X_l\sim N(0,1)$. To approach an estimation of $\rho$, we first calculate the mean value and the second moment of $\Phi(\frac{X_l}{\sqrt{1+\rho^2}})$.

\begin{defn}
Suppose $X \sim N(0,1)$ and $\Phi(\cdot)$ is the cumulative distribution function of standard normal distribution. Then we call $\xi=\Phi(\frac{X}{\sqrt{1+\rho^2}})$ a win rate random variable with parameter $\rho\geqslant0$, and simply write $\xi \sim Wr(\rho)$.
\end{defn}

\begin{prop}
Suppose $\xi \sim Wr(\rho),\Phi(\cdot)\ and\ p(\cdot)$ are the cumulative distribution function and probability density function of the standard normal distribution respectively. Then,\\
(1) the mean value of $\xi$
\begin{equation}
\mathbb{E}\xi =\frac{1}{2},
\end{equation}
(2) the second moment of $\xi$ 
\begin{equation}
\mathbb{E} \xi^2=\frac{1}{2}-\frac{1}{\pi}arctan\sqrt{\frac{1+\rho^2}{3+\rho^2}}.
\end{equation}
\label{thm moment}
\end{prop}

\begin{proof}
(1) An immediate computation shows that
\begin{equation}
\begin{split}
\mathbb{E}\xi&=\int^{+\infty}_{-\infty}p(x)\Phi\bigl({\frac{x}{\sqrt{1+\rho^2}}}\bigr)\dif x\\
&=\int^{+\infty}_{0}p(x)\left(\Phi\bigl(\frac{x}{\sqrt{1+\rho^2}}\bigr)+\Phi\bigl(\frac{-x}{\sqrt{1+\rho^2}}\bigr)\right)\dif x\\
&=\int^{+\infty}_{0}p(x)\dif x =\frac{1}{2}.
\end{split}
\end{equation}
(2) Define
\begin{equation}
\begin{split}
f(a)&=\int_{-\infty}^{+\infty}p(x)\Phi^2\bigl(\frac{x}{a}\bigr)\dif x\\
&=a\int_{-\infty}^{+\infty}p(ax)\Phi^2(x)\dif x =\int_{-\infty}^{+\infty}\Phi^2(x)\dif \Phi(ax)\\
&=1-2\int_{-\infty}^{+\infty}p(x)\Phi(x)\Phi(ax)\dif x.
\end{split}
\label{e func g def}
\end{equation}
Since $g(a,x)=p(x)\Phi(x)\Phi(ax) \geqslant 0 $ is continuously differentiable on $[0,+\infty)\times \mathbb{R}$ and 
\begin{equation}
\int_{-\infty}^{+\infty}g(a,x)\dif x\leqslant \int _{-\infty}^{+\infty}p(x)\dif x=1< \infty,
\end{equation}
\begin{equation}
\int_{-\infty}^{+\infty}\left|\frac{\partial}{\partial a}g(a,x)\right|\dif x\leqslant\frac{1}{\pi(1+a^2)}\int_{0}^{+\infty}e^{-t}\dif t=\frac{1}{\pi(1+a^2)}<\infty,
\end{equation}
the integral $\int_{0}^{+\infty}\frac{\partial}{\partial a}g(a,x)\dif x$ is uniformly convergent for $a\in [0,+\infty)$. Hence,
\begin{equation}
\begin{split}
\frac{\dif}{\dif a}\int_{-\infty}^{+\infty}p(x)\Phi(x)\Phi(ax)\dif x&=\int_{-\infty}^{+\infty}\frac{\partial}{\partial a}g(a,x)\dif x\\
&=-\frac{1}{2\pi(1+a^2)}\int_{-\infty}^{+\infty}\Phi(x)\dif e^{-\frac{1+a^2}{2}x^2}\\
&=\frac{1}{2\pi(1+a^2)}\frac{1}{\sqrt{2\pi}}\int_{-\infty}^{+\infty}e^{-\frac{2+a^2}{2}x^2}\dif x\\
&=\frac{1}{2\pi}\frac{1}{(1+a^2)\sqrt{2+a^2}}.
\end{split}
\label{e diff int}
\end{equation}
Also,
\begin{equation}
\int_{-\infty}^{+\infty}p(x)\Phi(x)\Phi(0)\dif x=\frac{1}{2}\int_{-\infty}^{+\infty}\Phi(x)\dif \Phi(x)=\frac{1}{4}.
\label{e int init val}
\end{equation}
By (\ref{e func g def}), (\ref{e diff int}) and (\ref{e int init val}), we have
\begin{equation}
\begin{split}
f(a)&=\frac{1}{2}-\frac{1}{\pi}\int_{0}^{a}\frac{1}{(1+u^2)\sqrt{2+u^2}}\dif u\\
&=\frac{1}{2}-\frac{1}{\pi}\int_{0}^{arctan\frac{a}{\sqrt{2}}}\frac{\sqrt{2}sec^2\theta \dif \theta}{\sqrt{2}(1+2tan^2\theta)sec\theta}\\
&=\frac{1}{2}-\frac{1}{\pi}\int_{0}^{arctan\frac{a}{\sqrt{2}}}\frac{\dif sin\theta}{1+sin^2\theta}\\
&=\frac{1}{2}-\frac{1}{\pi}arctan\frac{a}{\sqrt{2+a^2}}.
\end{split}
\end{equation}
Therefore,
\begin{equation}
\mathbb{E}\xi^2=\int_{-\infty}^{+\infty}p(x)\Phi^2\bigl(\frac{x}{\sqrt{1+\rho^2}}\bigr)\dif x=f(\sqrt{1+\rho^2})=\frac{1}{2}-\frac{1}{\pi}arctan\sqrt{\frac{1+\rho^2}{3+\rho^2}}.
\end{equation}
\end{proof}

With the help of the second moment of a win rate random variable, we can directly find a moment estimation of the coefficient of fluctuations.

\begin{thm}
Assume in a 1-1 arena with uniform fluctuations, $m$ players are sampled randomly and their results $I_{lk}=\mathbf{1}$\{The $l$-th player wins his $k$-th round\} form an $m\times n$ sample matrix
\[I=\left(\begin{array}{cccc}
I_{11}&I_{12}&\cdots&I_{1n}\\
I_{21}&I_{22}&\cdots&I_{2n}\\
\vdots&\vdots&\ddots&\vdots\\
I_{m1}&I_{m2}&\cdots&I_{mn}\\
\end{array}\right).\]
Then, the estimator 
\begin{equation}
\hat{\rho}=\sqrt{\frac{3-tan^2\pi T}{tan^2\pi T-1}}
\label{e esti of infty}
\end{equation}
is a strongly consistent estimation of the coefficient of fluctuations $\rho$, where
\begin{equation}
T=\frac{1}{n-1}\left(\frac{1}{mn}\sum_{l=1}^{m}Y_{l}^2-\frac{1}{2}\right),\qquad Y_l=\sum_{k=1}^{n}I_{lk},\qquad l=1,2,\cdots,m.
\end{equation}
\label{t stat1}
\end{thm}

\begin{proof}
First, we prove $I_{l_1k_1}$ and $I_{l_2k_2}$ are independent for any $l_1\neq l_2$. By Assumption (A3') and Definition \ref{d matc infty}, we have 
$$X_{l_1},\epsilon_{l_1,k_1},X_{\sigma_{k_1}(l_1)},\epsilon_{\sigma_{k_1}(l_1),k_1},X_{l_2},\epsilon_{l_2,k_2},X_{\sigma_{k_2}(l_2)},\epsilon_{\sigma_{k_2}(l_2),k_2}$$ are mutually independent. Also,
$$I_{l_1k_1}=\mathbf{1}\{X_{l_1}+\frac{\rho}{\sqrt{2}}\epsilon_{l_1,k_1}>X_{\sigma_{k_1}(l_1)}+\frac{\rho}{\sqrt{2}}\epsilon_{\sigma_{k_1}(l_1),k_1}\},$$
$$I_{l_2k_2}=\mathbf{1}\{X_{l_2}+\frac{\rho}{\sqrt{2}}\epsilon_{l_2,k_2}>X_{\sigma_{k_2}(l_2)}+\frac{\rho}{\sqrt{2}}\epsilon_{\sigma_{k_2}(l_2),k_2}\},$$
we have $I_{l_1k_1},I_{l_2k_2}(l_1\neq l_2)$ are independent. Therefore,
$$(I_{11},I_{12},\cdots,I_{1n}),(I_{21},I_{22},\cdots,I_{2n}),\cdots,(I_{m1},I_{m2},\cdots,I_{mn})$$ are mutually independent. Let $Y_l=\sum\limits_{k=1}^{n}I_{lk}(l=1,2,\cdots,m)$, then $Y_1,Y_2,\cdots,Y_m$ are independent. According to Theorem \ref{t cond dist}, we have
$$Y_l|X_l=\lambda\sim B\left(n,\Phi\bigl(\frac{\lambda}{\sqrt{1+\rho^2}}\bigr)\right).$$
Applying Proposition \ref{thm moment} yields
\begin{equation}
\begin{split}
\mathbb{E}(Y_l^2)&=n\mathbb{E}\Phi\bigl(\frac{X_l}{\sqrt{1+\rho^2}}\bigr)+(n^2-n)\mathbb{E}\Phi^2\bigl(\frac{X_l}{\sqrt{1+\rho^2}}\bigr)\\
&=\frac{n^2}{2}-\frac{n^2-n}{\pi}arctan\sqrt{\frac{1+\rho^2}{3+\rho^2}}.
\end{split}
\label{e expe}
\end{equation}
Since $Y_1,Y_2,\cdots,Y_m$ are mutually independent and identically distributed, applying strong law of large numbers gives
\begin{equation}
\frac{Y_1^2+Y_2^2+\cdots+Y_m^2}{m}\longrightarrow \frac{n^2}{2}-\frac{n^2-n}{\pi}arctan\sqrt{\frac{1+\rho^2}{3+\rho^2}}\quad(m\rightarrow \infty)\quad a.s.
\end{equation}
\begin{equation}
\hat{\rho}=\sqrt{\frac{3-tan^2\pi T}{tan^2\pi T-1}}\longrightarrow \rho \quad (m\rightarrow \infty)\quad a.s.
\end{equation}
In conclusion, $\hat{\rho}$ is a strongly consistent estimator of $\rho$.
\end{proof}

\section{Arena model with fluctuations in the finite case}

Part of the Assumption (A1) is impractical since the population of players is a countably infinite set since the number of players in a game is always finite. So we update the Assumption (A1) into

(A1') In an arena, an infinite number of \emph{runs} can be held among a fixed group of individuals, and these individuals are called \emph{players}. All players constitute a finite set $A_{0,0}^{q}=\{a_1,a_2,\cdots,a_M\}$, where $a_l$ is the $l$-th player, $M$ is even and $q=1,2,\cdots$.

A question is naturally posed that does the above estimator still perform well under the new assumption? The answer is not straightforward for two reasons. First, the normality assumption in Assumption (A2') is almost meaningless and the statement
$$I_{l1}|X_l=\lambda,I_{l2}|X_l=\lambda,\cdots,I_{ln}|X_l=\lambda\ i.i.d\sim Bernoulli\left(\Phi\bigl(\frac{\lambda}{\sqrt{1+\rho^2}}\bigr)\right)$$
no longer holds because in this case, the strengths of a player's opponents are some constants instead of a continuous random variable. Now we reconsider the value of $\mathbb{P}(I_{11}=I_{12}=1|X_1=\lambda)$ when the total number of players is an even number $M<\infty$. Let us compute it with an extreme example. Suppose only two players A and B participate in a 1-1 arena and $\rho$ is zero. The strength of player A's opponent will be a constant during $n$ rounds although it normally distributes in the first round. Therefore,
\[
\begin{split}
\Phi(\lambda)&=\mathbb{P}(I_{11}=I_{12}=1|X_1=\lambda)=\mathbb{P}(X_2<X_1|X_1=\lambda)\\
&\neq\mathbb{P}(I_{11}=1|X_1=\lambda)\mathbb{P}(I_{12}=1|X_2=\lambda)=\bigl(\Phi(\lambda)\bigr)^2.
\end{split}
\]
In addition, the total number of every player's wins
$$Y_l=\sum_{k=1}^{n}I_{lk},\qquad l=1,2,\cdots,m,$$
are not independent obviously and there is an identity that
\begin{equation}
\sum_{l=1}^{M}Y_l=\frac{Mn}{2}.
\end{equation}

Thanks to our definition of random matching on a countably infinite set, we guess naturally from Theorem \ref{t cond dist} that
\begin{equation}
\lim\limits_{M\rightarrow\infty}\mathbb{P}(I_{11}=1,I_{12}=1|X_1=\lambda)=\lim\limits_{M\rightarrow\infty}\bigl[\mathbb{P}(I_{11}=1|X_1=\lambda)\mathbb{P}(I_{12}=1|X_1=\lambda)\bigr].
\end{equation}
In other words,
\begin{equation}
\mathbb{P}(I_{11}=1,I_{12}=1|X_1=\lambda)=\mathbb{P}(I_{11}=1|X_1=\lambda)\mathbb{P}(I_{12}=1|X_1=\lambda)+o(1),\ \mathrm{as}\ M\rightarrow\infty.
\end{equation}
In order to prove an estimator similar to (\ref{e esti of infty}) is consistent, $\mathbb{E}Y_l^2$ and $Var(\sum_{l=1}^{M}Y_l^2)$ with respect to $M$ are required. Taking $M$ into consideration, we assume $I_{lk}^{(M)}$ and $Y_{l}^{(M)}=\sum_{k=1}^{n}I_{lk}^{(M)}$ to be the result of the $l$-th player in his $k$-th round and the total number of wins of the $l$-th player repectively when the population of the players is a finite set $A_M=\{a_1,a_2,\cdots,a_M\}$ and $I_{lk}$ and $Y_{l}=\sum_{k=1}^{n}I_{lk}$ the ones as we notate in Section 5.

\begin{lemma}
As defined before,
$$Y_l^{(M)}\overset{\text{d}}{\longrightarrow}Y_l,\ \mathrm{as}\ M\rightarrow\infty.$$
\label{t dist conv}
\end{lemma}

\begin{proof}
For a given even number $M$ and $1\leqslant l\leqslant M$, construct an event
\begin{equation}
\mathcal{A}=\bigcup_{i=1}^{n-1}\bigcup_{j=i+1}^{n}\{\sigma_i(l)=\sigma_j(l)\}.
\end{equation}
It is easy to see that $\mathcal{A}$ does not happen iff no player competes with the $l$-th player more than once in $n$ rounds. On one hand, since $\mathcal{A}$ and $X_l$ are independent, for arbitrary $\lambda\in \mathbb{R}$ we have 
\begin{equation}
\begin{split}
&\mathbb{P}\left(I_{l1}^{(M)}=1,I_{l2}^{(M)}=1,\cdots,I_{ln}^{(M)}=1\middle|X_l=\lambda\right)\\
\geqslant &\mathbb{P}\left(I_{l1}^{(M)}=1,I_{l2}^{(M)}=1,\cdots,I_{ln}^{(M)}=1,\bar{\mathcal{A}}\middle|X_l=\lambda\right)\\
=&\mathbb{P}\left(I_{l1}^{(M)}=1,I_{l2}^{(M)}=1,\cdots,I_{ln}^{(M)}=1\middle|\bar{\mathcal{A}},X_l=\lambda\right)\mathbb{P}\left(\bar{\mathcal{A}}\right).
\end{split}
\label{e cond prob geq1}
\end{equation}
By Assumption (A3') and Definition \ref{d matc fty}, $\sigma_i(l)$ and $\sigma_j(l)$ are independent and 
$$\mathbb{P}\bigl(\sigma_i(l)=l'\bigr)=\frac{1}{M-1},\qquad \forall l'\in A_M\backslash\{a_l\}.$$
Therefore,
\begin{equation}
\mathbb{P}(\mathcal{A})\leqslant\sum\limits_{i=1}^{n-1}\sum\limits_{j=i+1}^{n}\mathbb{P}\bigl(\sigma_i(l)=\sigma_j(l)\bigr)=\frac{n(n-1)}{2}\frac{1}{M-1}.
\end{equation}
Provided that M is large enough,
\begin{equation}
\mathbb{P}(\bar{\mathcal{A}})=1-\mathbb{P}(\mathcal{A})\geqslant 1-\frac{n(n-1)}{2}\frac{1}{M-1}>0.
\label{e prob of event geq}
\end{equation}
According to Assumption (A2'), we have
\begin{equation}
\begin{split}
&\mathbb{P}\left(I_{l1}^{(M)}=1,I_{l2}^{(M)}=1,\cdots,I_{ln}^{(M)}=1\middle|\bar{\mathcal{A}},X_l=\lambda\right)\\
=&\mathbb{P}\left(X_{\sigma_1(l)}+\frac{\rho}{\sqrt{2}}\epsilon_{\sigma_1(l),1}-\frac{\rho}{\sqrt{2}}\epsilon_{l,1}<\lambda,\cdots,X_{\sigma_n(l)}+\frac{\rho}{\sqrt{2}}\epsilon_{\sigma_n(l),n}-\frac{\rho}{\sqrt{2}}\epsilon_{l,n}<\lambda\middle|\bar{\mathcal{A}},X_l=\lambda\right),
\end{split}
\label{e cond prob meaning}
\end{equation}
and in the condition that $\forall i\neq j,\sigma_i(l)\neq\sigma_j(l)$, $X_l$,$X_{\sigma_i(l)}$ and $X_{\sigma_j(l)}$ are mutually independent. As a result,
$$X_{\sigma_1(l)}+\frac{\rho}{\sqrt{2}}\epsilon_{\sigma_1(l),1}-\frac{\rho}{\sqrt{2}}\epsilon_{l,1},\cdots,X_{\sigma_n(l)}+\frac{\rho}{\sqrt{2}}\epsilon_{\sigma_n(l),n}-\frac{\rho}{\sqrt{2}}\epsilon_{l,n}\ i.i.d\sim N(0,\sqrt{1+\rho^2}),$$
which are independent of $X_l$, given $\bar{\mathcal{A}}$. Combining it with (\ref{e defn Ilk}), (\ref{e cond prob geq1}), (\ref{e prob of event geq}) and (\ref{e cond prob meaning}) yields
\begin{equation}
\mathbb{P}\left(I_{l1}^{(M)}=1,I_{l2}^{(M)}=1,\cdots,I_{ln}^{(M)}=1\middle|X_l=\lambda\right)\geqslant \bigl(1-\frac{n(n-1)}{2}\frac{1}{M-1}\bigr)\Bigl[\Phi\bigl(\frac{\lambda}{\sqrt{1+\rho^2}}\bigr)\Bigr]^n.
\label{e cond prob geq2}
\end{equation}
On the other hand,
\begin{equation}
\begin{split}
&\mathbb{P}\left(I_{l1}^{(M)}=1,I_{l2}^{(M)}=1,\cdots,I_{ln}^{(M)}=1\middle|X_l=\lambda\right)\\
\leqslant &\mathbb{P}\left(I_{l1}^{(M)}=1,I_{l2}^{(M)}=1,\cdots,I_{ln}^{(M)}=1\middle|\bar{\mathcal{A}},X_l=\lambda\right)\mathbb{P}(\bar{\mathcal{A}})+\mathbb{P}(\mathcal{A})\\
\leqslant &\Bigl[\Phi\bigl(\frac{\lambda}{\sqrt{1+\rho^2}}\bigr)\Bigr]^n+\frac{n(n-1)}{2}\frac{1}{M-1}.
\end{split}
\label{e cond prob leq1}
\end{equation}
Applying the Squeezing Theorem to (\ref{e cond prob geq2}) and (\ref{e cond prob leq1}) yields
\begin{equation}
\begin{split}
&\lim\limits_{M\rightarrow\infty}\mathbb{P}\left(I_{l1}^{(M)}=1,I_{l2}^{(M)}=1,\cdots,I_{ln}^{(M)}=1\middle|X_l=\lambda\right)\\
=&\mathbb{P}\left(I_{l1}=1,I_{l2}=1,\cdots,I_{ln}=1\middle|X_l=\lambda\right)=\Bigl[\Phi\bigl(\frac{\lambda}{\sqrt{1+\rho^2}}\bigr)\Bigr]^n.
\end{split}
\end{equation}
If $I_{l1},I_{l2},\cdots,I_{ln}$ take other values, it suffices to flip the inequality sign in (\ref{e cond prob meaning}). Hence,
\begin{equation}
\left(I_{l1}^{(M)},I_{l2}^{(M)},\cdots,I_{ln}^{(M)}\right)\overset{\text{d}}{\longrightarrow}(I_{l1},I_{l2},\cdots,I_{ln}),\quad \mathrm{as}\ M\rightarrow\infty,
\end{equation}
\begin{equation}
Y_l^{(M)}=\sum\limits_{k=1}^{n}I_{lk}^{(M)}\overset{\text{d}}{\longrightarrow}Y_l=\sum\limits_{k=1}^{n}I_{ln},\quad \mathrm{as}\ M\rightarrow\infty.
\end{equation}
\end{proof}

\begin{thm}
Assume in a 1-1 arena with uniform fluctuations, the total number of players is an even number $M$ and their results $I_{lk}^{(M)}=\mathbf{1}$\{The $l$-th player wins his $k$-th rounds\} form an $M\times n$ sample matrix
\[I_M=\left(\begin{array}{cccc}
I_{11}^{(M)}&I_{12}^{(M)}&\cdots&I_{1n}^{(M)}\\
I_{21}^{(M)}&I_{22}^{(M)}&\cdots&I_{2n}^{(M)}\\
\vdots&\vdots&\ddots&\vdots\\
I_{M1}^{(M)}&I_{M2}^{(M)}&\cdots&I_{Mn}^{(M)}\\
\end{array}\right),\]
where $n\geqslant 2$. Then, the estimator 
\begin{equation}
\hat{\rho_M}=\sqrt{\frac{3-tan^2\pi T_M}{tan^2\pi T_M-1}}
\label{estimator of finite}
\end{equation}
is a consistent estimator of the coefficient of fluctuations $\rho$, where 
\begin{equation}
T_M=\frac{1}{n-1}\left(\frac{1}{Mn}\sum_{l=1}^{M}(Y_{l}^{(M)})^2-\frac{1}{2}\right),\qquad Y_l^{(M)}=\sum_{k=1}^{n}I_{lk}^{(M)},\qquad l=1,2,\cdots,M.
\end{equation}
\label{t stat2}
\end{thm}

\begin{proof}
See Appendix \ref{a}.
\end{proof}

\section{Tests and applications}

In the first two parts, the arena model without fluctuations is evaluated on practical data from FIFA World Cup and Hearthstone. In the third part, we conduct a simulation test for the estimator in Theorem \ref{t stat2}. Then we forge a connection between the arena model with fluctuations and classic models on paired comparisons. Finally, a metric is introduced to quantify the uncertainty in competitions.

\subsection{Test with FIFA World Cup data}

As we all know, knockouts are indispensable for the majority of sports, where our arena model has a good performance. For instance, in every FIFA World Cup, the top two teams of each group advance to the round of 16 and then compete on a knockout stage. The arena model provides a direct way to make inferences for each team. We first stated that we are by no means predicting the results of the next World Cup here, but demonstrating how to apply the method to data produced from knockouts. The results of Brazil, Italy, Argentina and Sweden over 80 years are listed below, where the number 0, 1, 2, 3, 4 and 5 means their final results are respectively ``fails to be top 16'', ``9th$\sim$16th'', ``5th$\sim$8th'', ``3th or 4th'', ``runner-up'' and ``champion''. Next, we use the data in Table 1 to estimate each team's future results and the data in Table 2 to test our estimations.

\begin{table}[!htb]
\centering
\caption{Data for train}\label{table1}
\begin{tabular}{ccccccccccc}
\toprule
Country& 1930& 1938& 1954& 1962& 1970& 1978& 1986& 1994& 2002& 2010\\
\midrule
Brazil& 3& 3& 2& 5& 5& 3& 2& 5& 5& 2\\
Italy& 0& 5& 1& 1& 4& 3& 1& 4& 1& 0\\
Argentina& 4& 0& 0& 1& 0& 5& 5& 1& 0& 2\\
Sweden& 0& 3& 0& 0& 1& 1& 0& 3& 1& 0\\
\bottomrule
\end{tabular}
\end{table}

\begin{table}[!htb]
\centering
\caption{Data for test}\label{table2}
\begin{tabular}{ccccccccccc}
\toprule
Country& 1934& 1950& 1958& 1966& 1974& 1982& 1990& 1998& 2006& 2014\\
\midrule
Brazil& 1& 4& 5& 1& 3& 2& 1& 4& 2& 3\\
Italy& 5& 2& 0& 1& 1& 5& 3& 2& 5& 0\\
Argentina& 1& 0& 1& 2& 2& 1& 4& 2& 2& 4\\
Sweden& 2& 3& 4& 0& 2& 0& 0& 0& 1& 0\\
\bottomrule
\end{tabular}
\end{table}

Even though the FIFA World Cup involves many random factors and the strengths of teams change occasionally, we view it as a 5-1 arena without fluctuations here to show some rough prediction results. According to the Theorem \ref{t post dist invari}, we can just assume $p(x)=1(0\leqslant x\leqslant 1)$ without loss of generality. Theorem \ref{t prob no infer1}, \ref{t dens of str}, \ref{t prob with infer} and \ref{t post dest} yield
\begin{equation}
\left\{\begin{aligned}
p_0&=1/2\\
p_1&=1/4\\
p_2&=1/8\\
p_3&=1/16\\
p_4&=1/32\\
p_5&=1/32\\
\end{aligned}
\right.\qquad and \qquad
\left\{\begin{aligned}
p_0(x)&=1-x\\
p_1(x)&=x(1-x^2)\\
p_2(x)&=x^3(1-x^4)\\
p_3(x)&=x^7(1-x^8)\\
p_4(x)&=x^{15}(1-x^{16})\\
p_5(x)&=x^{31},\\
\end{aligned}\right.
\end{equation}
where 
\begin{equation}
p_i=\begin{cases}
\mathbb{P}(\mathcal{A}_{i,1}),&0\leqslant i\leqslant 4,\\
\mathbb{P}(\mathcal{A}_{5,0}),&i=5,
\end{cases}\quad and \quad p_i(x)=\begin{cases}
\mathbb{P}(\mathcal{A}_{i,1}|X=x),&0\leqslant i\leqslant 4,\\
\mathbb{P}(\mathcal{A}_{5,0}|X=x),&i=5,
\end{cases}
\end{equation}
$\mathcal{A}_{i,j}$ and $X$ follow Theorem \ref{t prob with infer}. Then we have
\begin{equation}
\mathbb{P}\bigl(\xi=k|(N_0,N_1,\cdots,N_5)\bigr)=\frac{\int_0^1p_k(x)\prod\limits_{i=0}^5[p_i(x)]^{N_i}\dif x}{\int_0^1\prod\limits_{i=0}^5[p_i(x)]^{N_i}\dif x},
\end{equation}
where $N_i$ and $\xi$ represent the times that the result of a team is $i$ in Table 1 and its future result respectively. As a comparison, the sample mean estimator is also computed.

\begin{table}[!htbp]
\centering
\begin{tabular}{|c|c|c|c|c|c|c|c|c|c|c|c|c|}
\hline
\multirow{2}*{$\xi$}& \multicolumn{3}{|c|}{Brazil}& \multicolumn{3}{|c|}{Italy}& \multicolumn{3}{|c|}{Argentina}& \multicolumn{3}{|c|}{Sweden}\\
\cline{2-13}
& F& P1& P2& F& P1& P2& F& P1& P2& F& P1& P2\\
\hline
0& 0& 0.04& 0& 0.2& 0.1& 0.2& 0.1& 0.09& 0.4& 0.4& 0.33& 0.5\\
\hline
1& 0.3& 0.08& 0& 0.2& 0.17& 0.4& 0.3& 0.15& 0.2& 0.2& 0.35& 0.3\\
\hline
2& 0.2& 0.13& 0.3& 0.2& 0.24& 0& 0.4& 0.23& 0.1& 0.2& 0.23& 0\\
\hline
3& 0.2& 0.20& 0.3& 0.1& 0.26& 0.1& 0& 0.26& 0& 0.1& 0.07& 0.2\\
\hline
4& 0.2& 0.24& 0& 0& 0.17& 0.2& 0.2& 0.19& 0.1& 0.1& 0& 0\\
\hline
5& 0.1& 0.31& 0.4& 0.3& 0.06& 0.1& 0& 0.08 & 0.2& 0& 0& 0\\
\hline
\end{tabular}
\caption{Results within different methods}
\end{table}

In Table 3,  the ``F'' refers to the frequency of different results of these four teams in Table 2. The ``P1'' and ``P2'' are the results of our model and sample average with the data in Table 1 respectively. Indeed, Assumptions (A1), (A2) and (A3) in our model do not hold in this example. In fact, the prediction of results of FIFA World Cup is extremely complex, and we cannot solve it simply by the current arena model. Nevertheless, the arena model without fluctuations enables us to quantify individuals' strengths easily, which may bring great convenience when applied in biology and machine learning.

\subsection{Application in Hearthstone}

Actually, our basic ideas are inspired by a game mode of \emph{Hearthstone}---a hot game of Blizzard---called \emph{``Arena''}, whose rules are as follows (for more details, see \cite{hearth1}).\\
(1) After paying the entry fee, the player will build a deck in a specific way, which has some uncertainty and is similar to (R1). (Of course, a deck differs from a random number.)\\
(2) Playing in Arena consists of a series of matches between Arena players, which usually works like (R3).\\
(3) During each Arena run, a player can suffer up to three losses. Once a player has lost three times or won 12 times, his run ends.

Since the way players build their decks is complicated, instead of applying arenas to this mode, we study a similar mode called ``Standard Brawliseum''. In this mode, players build their decks by themselves, and our Assumption (A1) holds approximately (for more details, see \cite{hearth2}). Some uncertainty does exist during each round of it, but we still regard ``Standard Brawliseum'' as a 12-3 arena without fluctuations here to give a simple estimation using the past match data.

Suppose a player have obtained 12-0, 10-3, 6-3, 12-2 in four runs with the same deck, now we want to give an estimation of the probability that he obtains different results with this deck in the future. By Theorem \ref{t prob no infer1} and \ref{t post dist invari}, we can express the posterior probability that the player will end with $(m,j)$ or $(i,n)$ in a run in
\begin{equation}
\binom{m+j-1}{m-1}(\frac{1}{2})^{m+j}\frac{\int_{0}^{1}p_{m,j}(t)\prod_{(k,l)\in\Lambda}\bigl[p_{k,l}(t)\bigr]^{N_{k,l}}\dif t}{\int_{0}^{1}\prod_{(k,l)\in\Lambda}\bigl[p_{k,l}(t)\bigr]^{N_{k,l}}\dif t}
\label{e apply1}
\end{equation}
and
\begin{equation}
\binom{n+i-1}{n-1}(\frac{1}{2})^{n+i}\frac{\int_{0}^{1}p_{i,n}(t)\prod_{(k,l)\in\Lambda}\bigl[p_{k,l}(t)\bigr]^{N_{k,l}}\dif t}{\int_{0}^{1}\prod_{(k,l)\in\Lambda}\bigl[p_{k,l}(t)\bigr]^{N_{k,l}}\dif t}
\label{e apply2}
\end{equation}
respectively, where $j=0,1,\cdots,n-1$, $i=0,1,\cdots,m-1$, $N_{k,l}$ refers to the times that the player's result is $(k,l)$ in the past data. The $\Lambda$, $p_{k,l}(\cdot)$ follow Definition \ref{d dist arena} and Theorem \ref{t dens of str} (choose $p(x)=\mathbf{1}\{0\leqslant x\leqslant 1\}$). Since $m=12$ is large, we use difference quotients and trapezoid formula to compute the density function and integral respectively. Plugging known quantity into (\ref{e apply1}) and (\ref{e apply2}) gives the table 4.

\begin{table}[!htbp]
\centering
\caption{Estimation of result}\label{table4}
\begin{tabular}{cccccccccccccccc}
\toprule
Result& 12-0& 12-1& 12-2& 11-3& 10-3& 9-3& 8-3& 7-3\\
\midrule
Probability& $3.0*10^{-4}$& $0.20$& $0.72$& $2.6*10^{-2}$& $5.5*10^{-3}$& $6.5*10^{-4}$& $4.2*10^{-5}$& $1.6*10^{-6}$\\
\bottomrule
\end{tabular}
\end{table}

Although we predict the results easily by applying the arena model without fluctuations, it seems not to fit his past results well. Fluctuations do exist! So much further research work is required to improve the arena model. For example, how to predict in an arena with fluctuations of more general form, and how can we judge whether there are fluctuations in an arena. We do not answer this question in this paper and will solve these problems in the future work.

\subsection{Simulations}

Here is a test for the estimation of $\rho$ in Theorem \ref{t stat2}. For $\rho=0.1,0.5,1,2,4$ and 6, $N=1000$ groups of random numbers are generated for $M=1024,n=8$ and $M=1024,n=16$ and $M=8192,n=8$ respectively. The simulative results of the mean and MSE of the estimator are listed below. In fact, the estimator is not valid for all samples, so we here set $\hat{\rho}=0$ if $T_M>\frac{1}{3}$ and $\hat{\rho}=10$ if $T_M<\frac{1}{4}$.
\begin{table}[!htbp]
\centering
\begin{tabular}{|c|c|c|c|c|c|c|}
\hline
\multicolumn{7}{|c|}{$N=1000$}\\ 
\hline
\multirow{2}*{$\rho$}& \multicolumn{2}{|c|}{M=1024,n=8}& \multicolumn{2}{|c|}{M=1024,n=16}& \multicolumn{2}{|c|}{M=8192,n=8}\\
\cline{2-7}
& E& MSE& E& MSE& E& MSE\\
\hline
0.1& 0.099& 0.011& 0.092& 0.007& 0.087& 0.004\\
\hline
0.5& 0.496& 0.004& 0.498& 0.002& 0.499& 0.0004\\
\hline
1& 0.997& 0.003& 1.001& 0.002& 0.999& 0.0004\\
\hline
2& 2.003& 0.014& 2.001& 0.006& 2.001& 0.002\\
\hline
4& 4.074& 0.232& 4.018& 0.072& 4.002& 0.024\\
\hline
6& 6.406& 3.597& 6.083& 0.614& 6.034& 0.188\\
\hline
\end{tabular}
\caption{Test of $\hat{\rho}$}
\end{table}

Evidently, our estimation of $\rho$ performs well in both accuracy and stability when $M$ and $n$ is relatively large. We can define the index of competition $\beta=\frac{1}{1+\rho}$, which is a better metric of the fluctuations of an arena in practice since it is a bijection from $[0,+\infty]$ to $[0,1]$ and can reduce relative error. Furthermore, $\beta$ has significance in sports that it will be close to 1 if a game has little risk and relies almost on players' capacity and 0 otherwise, as a measurement of the stability and fairness of a game.

\subsection{About paired comparison and Bradley-Terry model}

In the Bradley-Terry model, the probability that object $i$ is judged to have more of an attribute than object $j$ is
\begin{equation}
\mathbb{P}(X_{ij}=1)=\frac{e^{\delta_i-\delta_j}}{1+e^{\delta_i-\delta_j}},
\end{equation}
where $\delta_{i}$ is the scale location of object $i$. In our model, if $\hat{\rho}$ is an estimator of $\rho$, we have 
\begin{equation}
\begin{split}
\mathbb{P}(\mathrm{player\ }a_i\mathrm{\ beats\ player\ }a_j|X_i=x_i,X_j=x_j)&=\mathbb{P}(x_i+\frac{\rho}{\sqrt{2}}\epsilon_i>x_j+\frac{\rho}{\sqrt{2}}\epsilon_j)\\
&\approx\mathbb{P}(x_i+\frac{\hat{\rho}}{\sqrt{2}}\epsilon_i>x_j+\frac{\hat{\rho}}{\sqrt{2}}\epsilon_j)\\
&=\Phi\left(\frac{x_i-x_j}{\hat{\rho}}\right).
\label{e prob arena}
\end{split}
\end{equation}
Here we regard $i$ and $j$ as two players and ``player $a_i$ beats player $a_j$'' means ``the agent prefers $i$ over $j$'' and try to find a direct association between the scale location $\delta_i$ and the strength $x_i$. Tocher introduce an approximation of standard normal distribution in \cite{AS} that 
\begin{equation}
\Phi(x)\approx\frac{e^{2kx}}{e^{2kx}+1},\quad k=\sqrt{\frac{2}{\pi}}.
\end{equation}
We now apply this to derive an estimation of $\delta_i$ with $\hat{\rho}$ in our model. Let $\mathbb{P}(X_{ij}=1)=\mathbb{P}(\mathrm{player\ }a_i\mathrm{\ beats\ player\ }a_j|X_i=x_i,X_j=x_j)$ and we have
\begin{equation}
\frac{e^{\delta_i-\delta_j}}{1+e^{\delta_i-\delta_j}}\approx\Phi\left(\frac{x_i-x_j}{\hat{\rho}}\right)\approx\frac{e^{2k(x_i-x_j)/\hat{\rho}}}{1+e^{2k(x_i-x_j)/\hat{\rho}}},
\end{equation}
which implies
\begin{equation}
\delta_i-\delta_j\approx\frac{2k(x_i-x_j)}{\hat{\rho}}.
\end{equation}
Hence
\begin{equation}
\frac{2k\hat{x_i}}{\hat{\rho}}=\frac{2k\sqrt{1+\hat{\rho}^2}}{\hat{\rho}}\Phi^{-1}\left(\frac{Y_i^{(M)}}{n}\right)
\end{equation}
in arena model can be regarded as an estimation of $\delta_i$ in Bradley-Terry model. Further, Glickman \cite{PEL} gives a reparameterized version of the Bradley-Terry model by assuming the prior distribution of a player's strength is 
\begin{equation}
\theta|\mu,\sigma^2\sim N(\mu,\sigma^2)
\label{e dist theta}
\end{equation}  
and the likelihood that $i$ beats $j$ in the $k$-th round is given by
\begin{equation}
\mathbb{P}(s_{ijk}=1|\theta_i,\theta_j)=\frac{10^{(\theta_i-\theta_j)/400}}{1+10^{(\theta_i-\theta_j)/400}}.
\label{e prob bt}
\end{equation}
To establish connection between it and our model, we assume that $\theta_i$ is independent of $\theta_j$ and $\sigma_i^2=\sigma_j^2=\sigma^2$ and the following proposition is required.

\begin{prop}
Suppose $\xi\sim N(\mu,\sigma^2)$ and $\Phi(\cdot)$ is the culmulative distribution function of the standard normal distribution. Then
\begin{equation}
\mathbb{E}\Phi(\xi)=\Phi\bigl(\frac{\mu}{\sqrt{1+\sigma^2}}\bigr).
\label{e exphi}
\end{equation}
\label{t exphi}
\end{prop} 

\begin{proof}
Define 
\begin{equation}
f(\mu)=\mathbb{E}\Phi(\xi)=\int_{-\infty}^{+\infty}\Phi(x)\frac{1}{\sqrt{2\pi}\sigma}e^{-\frac{(x-\mu)^2}{2\sigma^2}}\dif x.
\label{e def f}
\end{equation}
Since 
\begin{equation}
g(\mu,x)=\frac{1}{\sqrt{2\pi}\sigma}\Phi(x)e^{-\frac{(x-\mu)^2}{2\sigma^2}}\geqslant 0
\end{equation}
is continuous differentiable on $\mathbb{R}^2$ and
\begin{equation}
\int_{-\infty}^{+\infty}g(\mu,x)\dif x\leqslant \int_{-\infty}^{+\infty}\frac{1}{\sqrt{2\pi}\sigma}e^{-\frac{(x-\mu)^2}{2\sigma^2}}\dif x=1<\infty,
\end{equation}
\begin{equation}
\int_{-\infty}^{+\infty}\left|\frac{\partial}{\partial\mu}g(\mu,x)\right|\dif x\leqslant \frac{2}{\sqrt{2\pi}\sigma}\int_{-\infty}^{+\infty}\frac{x-\mu}{\sigma^2}e^{-\frac{(x-\mu)^2}{2\sigma^2}}\dif x=\frac{2}{\sqrt{2\pi}\sigma}<\infty,
\end{equation}
the integral $\int_{-\infty}^{+\infty}\frac{\partial}{\partial\mu}g(\mu,x)\dif x$ is uniformly convergent for $\mu\in\mathbb{R}$. Hence,
\begin{equation}
\begin{split}
f'(\mu)&=\int_{-\infty}^{+\infty}\frac{\partial}{\partial\mu}g(\mu,x)\dif x\\
&=-\frac{1}{\sqrt{2\pi}\sigma}\int_{-\infty}^{+\infty}\Phi(x)\dif e^{-\frac{(x-\mu)^2}{2\sigma^2}}\\
&=\frac{1}{2\pi\sigma}\int_{-\infty}^{+\infty}e^{-\frac{(1+\sigma^2)x^2-2\mu x+\mu^2}{2\sigma^2}}\dif x\\
&=\frac{1}{\sqrt{2\pi}\sqrt{1+\sigma^2}}e^{-\frac{\mu^2}{2(1+\sigma^2)}}.
\end{split}
\label{e fprime}
\end{equation}
Also,
\begin{equation}
f(0)=\frac{1}{\sqrt{2\pi}\sigma}\int_{-\infty}^{+\infty}\Phi(x)e^{-\frac{x^2}{2\sigma^2}}\dif x=\frac{1}{2}.
\label{e f0}
\end{equation}
By (\ref{e def f}), (\ref{e fprime}) and (\ref{e f0}), we have
\begin{equation}
\begin{split}
f(\mu)&=\frac{1}{2}+\int_{0}^{\mu}\frac{1}{\sqrt{2\pi}\sqrt{1+\sigma^2}}e^{-\frac{x^2}{2(1+\sigma^2)}}\dif x\\
&=\Phi\bigl(\frac{\mu}{\sqrt{1+\sigma^2}}\bigr).
\end{split}
\end{equation}
\end{proof}
According to (\ref{e prob arena}), (\ref{e dist theta}) and (\ref{e prob bt}), the above theorem gives
\begin{equation}
\begin{split}
\mathbb{P}(s_{ijk}=1)&=\mathbb{E}\frac{10^{(\theta_i-\theta_j)/400}}{1+10^{(\theta_i-\theta_j)/40}}\\
&\approx\mathbb{E}\Phi\bigl(\frac{\theta_i-\theta_j}{800k}ln10\bigr)\\
&=\Phi\biggl(\frac{(\mu_i-\mu_j)\frac{ln10}{800k}}{\sqrt{1+2\sigma^2(\frac{ln10}{800k})^2}}\biggr).
\end{split}
\end{equation}
Let $\mathbb{P}(s_{ijk}=1)=\mathbb{P}(a_i\mathrm{\ beats\ }a_j)$, we obtain an approximation that
\begin{equation}
\frac{(\mu_i-\mu_j)\frac{ln10}{800k}}{\sqrt{1+2\sigma^2(\frac{ln10}{800k})^2}}\approx\frac{x_i-x_j}{\rho}.
\end{equation}
If we already have an estimation of $x_i,x_j$ and $\rho$, namely $\hat{x_i},\hat{x_j}$ and $\hat{\rho}$, since $|\mu_i-\mu_j|\leqslant 3000$ and $\mathbb{P}(|x_i-x_j|\leqslant\frac{15}{4k})>0.99$ (the $\frac{15}{4k}$ is chosen subjectively and is supposed to be determined by actual data), we can estimate $(\mu_i,\mu_j,\sigma^2)$ by
\begin{equation}
\left\{\begin{aligned}
\hat{\mu_i}&=800k\hat{x_i}\\
\hat{\mu_j}&=800k\hat{x_j}\\
\hat{\sigma}^2&=320000k^2\left(\hat{\rho}^2-\frac{1}{ln^{2}10}\right).
\end{aligned}
\right.
\end{equation}
The inverse estimation using $\hat{\mu_i},\hat{\mu_j},\hat{\sigma}^2$ is also available.

\subsection{A metric of chaos}

For any given results in a 1-1 arena with uniform fluctuations, its coefficient of fluctuations can be computed by (\ref{estimator of finite}), which represents the instability and volatility of a competition. In a broad sense, for a matrix $A\in\mathbb{B}^{M\times n}$ which is made up of dummy variables and satisfies
\begin{equation}
\mathbf{1}_M'A=\frac{M}{2}\mathbf{1}_n'\ or\ \mathbf{1}_M'A\approx\frac{M}{2}\mathbf{1}_n',
\label{e defn matrix}
\end{equation}
$$T(A)=\frac{1}{n-1}\left(\frac{\mathbf{1}_n'A'A\mathbf{1}_n}{Mn}-\frac{1}{2}\right)\ or\ \rho(A)=\sqrt{\frac{3-tan^2\pi T(A)}{tan^2\pi T(A)-1}}\ or\ \beta(A)=\frac{1}{1+\rho(A)}$$
can be viewed as metrics of instability and randomness of a matrix $A$, based on the competition. We call any logical matrix satisfying (\ref{e defn matrix}) a win-loss matrix and $\beta(A)$ the competition index of $A$. As we stated before, for a win-loss matrix $A$, if $\beta(A)$ is close to 0, there is irrelevance among rows of $A$ in the sense of competition; if $\beta(A)$ is close to $1$, the competitive relationship among its rows exists, which is valuable when analyzing dummy variables. For instance, consider a matrix
\begin{equation}
A_n=\left(\begin{matrix}0& 0& \cdots& 0& 0\\
0& 0& \cdots& 0& 1\\
0& 0& \cdots& 1& 0\\
0& 0& \cdots& 1& 1\\
\vdots& \vdots& \ddots& \vdots& \vdots\\
1& 1& \cdots& 1& 0\\
1& 1& \cdots& 1& 1
\end{matrix}\right)\in\mathbb{B}^{2^n\times n},
\end{equation}
whose rows are $0,1,\cdots,M-1=2^n-1$ respectively, expressed in the binary numeral system with digits $n$. By simple computation, we have 
\begin{equation}
\begin{split}
\mathbf{1}_{M}'A_n&=2^{n-1}\mathbf{1}_n',\\
\mathbf{1}_n'A_n'A_n\mathbf{1}_n=\sum\limits_{k=0}^{n}k^2\binom{n}{k}&=n\left[2^{n-1}+(n-1)2^{n-2}\right],
\end{split}
\end{equation}
which means that $A_n$ is a win-loss matrix and its competition index $\beta(A_n)=0$; that is to say there is no competitive relationship between $A_n$.

Here we only apply these measurements to a specific matrix. They can serve as metrics of the degree of chaos when analyzing several groups of empirical binary data, which is likely to represent some quantities significant but still unknown in biology and machine learning.

\section{Conclusions}

Firstly, we introduce arenas without fluctuations to describe competitions and give an invariant estimator. These work made up the framework of inferences about comparisons without ratings. Secondly, we provide a simple estimator to detect to which extent the outcomes of competitions are determined by ``chance'' as opposed to ``skill''. It should be stressed here that we concern about the cases that the number of players is large enough. These conclusions are also based on this assumption.

In this paper, we only study paired competitions that there is only one winner in each round consisting of 2 players. Some further study can be done about the arena where $p$ individuals win out from $q$ individuals. Besides, we only study the simplest case of arena model with fluctuations. There is still much challenging generalization work to be done to improve the arena model.

\section*{\centering{Appendix A: Proof of Theorem 5.2}}

\begin{proof}
Obviously, for arbitrary $M\in2\mathbb{N}$, $Y_1^{(M)},\cdots,Y_n^{(M)}$ are distributed identically, as well as $Y_1,\cdots,Y_n$. Suppose $\mathbb{E}(Y_1^{(M)})^2=\mu_M$ and $\mathbb{E}Y_1^2=\mu$. Then for arbitrary $\epsilon>0$, Chebyshev's inequality gives
\begin{equation}
\begin{split}
&\mathbb{P}\left(\left|\frac{(Y_1^{(M)})^2+\cdots+(Y_M^{(M)})^2}{M}-\mu\right|\leqslant \epsilon\right)\\
\leqslant&\mathbb{P}\left(\left|\frac{(Y_1^{(M)})^2+\cdots+(Y_M^{(M)})^2}{M}-\mu_M\right|\leqslant \frac{\epsilon}{2}\right)+\mathbb{P}(|\mu_M-\mu|\leqslant\frac{\epsilon}{2})\\
\leqslant&\frac{4}{M^2\epsilon^2}Var\left[(Y_1^{(M)})^2+\cdots+(Y_M^{(M)})^2\right]+\mathbb{P}(|\mu_M-\mu|\leqslant\frac{\epsilon}{2})\\
=&\frac{4}{M^2\epsilon^2}\Bigl\{\sum_{l=1}^{M}Var(Y_l^{(M)})^2+2\sum_{1\leqslant i<j\leqslant M}Cov\left[(Y_i^{(M)})^2,(Y_j^{(M)})^2\right]\Bigr\}+\mathbb{P}(|\mu_M-\mu|\leqslant\frac{\epsilon}{2}).
\end{split}
\label{e prob epsilon leq}
\end{equation}
For any fixed positive integers $n\geqslant 2$,
\begin{equation}
\forall M\in2\mathbb{N},Var(Y_l^{(M)})^2\leqslant \mathbb{E}(Y_l^{(M)})^4\leqslant n^4.
\end{equation}
By Lemma \ref{t dist conv} and Helly's second theorem,
\begin{equation}
\mu_M=\mathbb{E}Y_l^{(M)}\rightarrow \mu=\mathbb{E}Y_l,\ \mathrm{as}\ M\rightarrow\infty.
\label{e expe conv}
\end{equation}
It remains to show that the upper bound of $Cov\left[(Y_i^{(M)})^2,(Y_j^{(M)})^2\right]$ can be controled by an $o(1)$ when $M\rightarrow\infty$. Due to the symmetry relation among
$$(I_{11},I_{12},\cdots,I_{1n}),(I_{21},I_{22},\cdots,I_{2n}),\cdots,(I_{M1},I_{M2},\cdots,I_{Mn}),$$
we can simplify $Cov\left[(Y_1^{(M)})^2,(Y_2^{(M)})^2\right]$ by
\begin{equation}
\begin{split}
&\mathbb{E}\left(I_{11}^{(M)}+\cdots+I_{1n}^{(M)}\right)^2\left(I_{21}^{(M)}+\cdots+I_{2n}^{(M)}\right)^2-\left[\mathbb{E}(Y_1^{(M)})^2\right]^2\\
=&\mathbb{E}\left[n(I_{11}^{(M)})^2+n(n-1)I_{11}^{(M)}I_{12}^{(M)}\right]\left(I_{21}^{(M)}+\cdots+I_{2n}^{(M)}\right)^2-\left[\mathbb{E}(Y_1^{(M)})^2\right]^2\\
=&n\mathbb{E}I_{11}^{(M)}\left(I_{21}^{(M)}+\cdots+I_{2n}^{(M)}\right)^2+n(n-1)\mathbb{E}I_{11}^{(M)}I_{12}^{(M)}\left(I_{21}^{(M)}+\cdots+I_{2n}^{(M)}\right)^2-\left[\mathbb{E}(Y_1^{(M)})^2\right]^2.
\end{split}
\label{e cov equal}
\end{equation}
For $k$=1,2,3,4, define
$$\mathcal{B}_k=\{a_1\mathrm{\ is\ the\ opponent\ of\ }a_2\mathrm{ \ in\ the\ }k\mathrm{-th\ round}\},$$
$$\bar{\mathcal{B}_k}=\{a_1\mathrm{\ is\ not\ the\ opponent\ of\ }a_2\mathrm{ \ in\ the\ }k\mathrm{-th\ round}\}.$$
Then we have
\begin{equation}
\begin{split}
&Cov\left[(Y_1^{(M)})^2,(Y_2^{(M)})^2\right]+\left[\mathbb{E}(Y_1^{(M)})^2\right]^2\\
=&n\mathbb{P}\left(I_{11}^{(M)}=I_{21}^{(M)}=1\right)+n(n-1)\mathbb{P}\left(I_{11}^{(M)}=I_{22}^{(M)}=1\right)+4n(n-1)\mathbb{P}\left(I_{11}^{(M)}=I_{21}^{(M)}=I_{22}^{(M)}=1\right)\\
&+2n(n-1)(n-2)\mathbb{P}\left(I_{11}^{(M)}=I_{22}^{(M)}=I_{23}^{(M)}=1\right)+2n(n-1)\mathbb{P}\left(I_{11}^{(M)}=I_{12}^{(M)}=I_{21}^{(M)}=I_{22}^{(M)}=1\right)\\
&+4n(n-1)(n-2)\mathbb{P}\left(I_{11}^{(M)}=I_{12}^{(M)}=I_{21}^{(M)}=I_{23}^{(M)}=1\right)\\
&+n(n-1)(n-2)(n-3)\mathbb{P}\left(I_{11}^{(M)}=I_{12}^{(M)}=I_{23}^{(M)}=I_{24}^{(M)}=1\right)\\
\leqslant&n\mathbb{P}(\bar{\mathcal{B}_1})\mathbb{P}\left(I_{11}^{(M)}=I_{21}^{(M)}=1\middle|\bar{\mathcal{B}_1}\right)+n(n-1)\left[\mathbb{P}(\mathcal{B}_2)+\mathbb{P}(\bar{\mathcal{B}_2})\mathbb{P}\left(I_{11}^{(M)}=I_{22}^{(M)}=1\middle|\bar{\mathcal{B}_2}\right)\right]\\
&+2n(n-1)(n-2)\left[\mathbb{P}(\mathcal{B}_1)+\mathbb{P}(\bar{\mathcal{B}_1})\mathbb{P}\left(I_{11}^{(M)}=I_{22}^{(M)}=I_{23}^{(M)}=1\middle|\bar{\mathcal{B}_1}\right)\right]\\
&+4n(n-1)\mathbb{P}(\bar{\mathcal{B}_1})\mathbb{P}\left(I_{11}^{(M)}=I_{21}^{(M)}=I_{22}^{(M)}=1\middle|\bar{\mathcal{B}_1}\right)\\
&+2n(n-1)\mathbb{P}(\bar{\mathcal{B}_1},\bar{\mathcal{B}_2})\mathbb{P}\left(I_{11}^{(M)}=I_{12}^{(M)}=I_{21}^{(M)}=I_{22}^{(M)}=1\middle|\bar{\mathcal{B}_1},\bar{\mathcal{B}_2}\right)\\
&+4n(n-1)(n-2)\left[\mathbb{P}(\mathcal{B}_3)+\mathbb{P}(\bar{\mathcal{B}_1},\bar{\mathcal{B}_3})\mathbb{P}\left(I_{11}^{(M)}=I_{12}^{(M)}=I_{21}^{(M)}=I_{23}^{(M)}=1\middle|\bar{\mathcal{B}_1},\bar{\mathcal{B}_3}\right)\right]\\
&+n(n-1)(n-2)(n-3)\left[\mathbb{P}(\mathcal{B}_3)+\mathbb{P}(\mathcal{B}_4)+\mathbb{P}(\bar{\mathcal{B}_3},\bar{\mathcal{B}_4})\mathbb{P}\left(I_{11}^{(M)}=I_{12}^{(M)}=I_{23}^{(M)}=I_{24}^{(M)}=1\middle|\bar{\mathcal{B}_3},\bar{\mathcal{B}_4}\right)\right].
\end{split}
\label{e cov leq}
\end{equation}
Firstly by Assumption (A1'), (A3') and Definition \ref{d matc fty}, we have
\begin{equation}
\mathbb{P}(\mathcal{B}_k)=\frac{1}{M-1}=o(1),\ \mathrm{as}\ M\rightarrow\infty,\quad k=1,2,3,4.
\label{e prob o}
\end{equation}
Secondly take $\mathbb{P}\left(I_{11}^{(M)}=I_{12}^{(M)}=I_{23}^{(M)}=I_{24}^{(M)}=1\middle|\bar{\mathcal{B}_3},\bar{\mathcal{B}_4}\right)$ as an example and we know that 
\[\begin{split}
&\mathbb{P}\left(I_{11}^{(M)}=I_{12}^{(M)}=I_{23}^{(M)}=I_{24}^{(M)}=1\middle|\bar{\mathcal{B}_3},\bar{\mathcal{B}_4}\right)\\
=&\mathbb{P}\left(I_{11}^{(M)}=I_{12}^{(M)}=I_{23}^{(M)}=I_{24}^{(M)}=1\middle|\sigma_3(2)\neq1,\sigma_4(2)\neq1\right)\\
=&\mathbb{P}\left(I_{23}^{(M)}=I_{24}^{(M)}=1\middle|I_{11}^{(M)}=I_{12}^{(M)}=1,\sigma_3(2)\neq1,\sigma_4(2)\neq1\right)\mathbb{P}\left(I_{11}^{(M)}=I_{12}^{(M)}=1\middle|\sigma_3(2)\neq1,\sigma_4(2)\neq1\right)\\
=&\biggl(\frac{1}{2}-\frac{1}{\pi}arctan\sqrt{\frac{1+\rho^2}{3+\rho^2}}\biggr)^2.
\end{split}\]
The last equation uses the fact that in the finite case, $X_l$ is independent of $I_{ik}$ if $i\neq l$. As a result,
\begin{equation}
\mathbb{P}\left(I_{11}^{(M)}=I_{12}^{(M)}=I_{23}^{(M)}=I_{24}^{(M)}=1\middle|\bar{\mathcal{B}_3},\bar{\mathcal{B}_4}\right)=\biggl(\frac{1}{2}-\frac{1}{\pi}arctan\sqrt{\frac{1+\rho^2}{3+\rho^2}}\biggr)^2+o(1),\ \mathrm{as}\ M\rightarrow\infty.
\label{e cond prob o1}
\end{equation}
By the same token, we have
\begin{equation}
\begin{split}
&\mathbb{P}\left(I_{11}^{(M)}=I_{22}^{(M)}=I_{23}^{(M)}=1\middle|\bar{\mathcal{B}_1}\right)\\
=&\mathbb{P}\left(I_{11}^{(M)}=I_{21}^{(M)}=I_{22}^{(M)}=1\middle|\bar{\mathcal{B}_1}\right)\\
=&\frac{1}{2}\biggl(\frac{1}{2}-\frac{1}{\pi}arctan\sqrt{\frac{1+\rho^2}{3+\rho^2}}\biggr)+o(1),\quad \mathrm{as}\ M\rightarrow\infty,
\end{split}
\label{e cond prob o2}
\end{equation}
\begin{equation}
\begin{split}
&\mathbb{P}\left(I_{11}^{(M)}=I_{12}^{(M)}=I_{21}^{(M)}=I_{22}^{(M)}=1\middle|\bar{\mathcal{B}_1},\bar{\mathcal{B}_2}\right)\\
=&\mathbb{P}\left(I_{11}^{(M)}=I_{12}^{(M)}=I_{21}^{(M)}=I_{23}^{(M)}=1\middle|\bar{\mathcal{B}_1},\bar{\mathcal{B}_3}\right)\\
=&\biggl(\frac{1}{2}-\frac{1}{\pi}arctan\sqrt{\frac{1+\rho^2}{3+\rho^2}}\biggr)^2+o(1),\quad \mathrm{as}\ M\rightarrow\infty,
\end{split}
\label{e cond prob o3}
\end{equation}
and
\begin{equation}
\mathbb{P}\left(I_{11}^{(M)}=I_{21}^{(M)}=1\middle|\bar{\mathcal{B}_1}\right)=\mathbb{P}\left(I_{11}^{(M)}=I_{22}^{(M)}=1\middle|\bar{\mathcal{B}_2}\right)=\frac{1}{4}+o(1),\quad \mathrm{as}\ M\rightarrow\infty.
\label{e cond prob o4}
\end{equation}
Hence by (\ref{e expe conv}), substituting (\ref{e expe}),(\ref{e prob o}),(\ref{e cond prob o1}),(\ref{e cond prob o2}),(\ref{e cond prob o3}),(\ref{e cond prob o4}) into (\ref{e cov leq}) yields
\begin{equation}
\begin{split}
&Cov\left[(Y_1^{(M)})^2,(Y_2^{(M)})^2\right]\\
\leqslant&\frac{n^2}{4}+n^2(n-1)\biggl(\frac{1}{2}-\frac{1}{\pi}arctan\sqrt{\frac{1+\rho^2}{3+\rho^2}}\biggr)+n^2(n-1)^2\biggl(\frac{1}{2}-\frac{1}{\pi}arctan\sqrt{\frac{1+\rho^2}{3+\rho^2}}\biggr)^2\\
&-\biggl(\frac{n^2}{2}-\frac{n^2-n}{\pi}arctan\sqrt{\frac{1+\rho^2}{3+\rho^2}}\biggr)^2+o(1)\\
=&o(1),\quad \mathrm{as}\ M\rightarrow\infty.
\end{split}
\label{e cov leq2}
\end{equation}
By (\ref{e prob epsilon leq}) and (\ref{e cov leq2}), we have
\begin{equation}
\lim\limits_{M\rightarrow\infty}\mathbb{P}\left(\left|\frac{(Y_1^{(M)})^2+\cdots+(Y_n^{(M)})^2}{M}-\mu\right|\leqslant \epsilon\right)=0.
\end{equation}
Therefore,
\begin{equation}
\frac{(Y_1^{(M)})^2+\cdots+(Y_M^{(M)})^2}{M}\stackrel{\text{P}}\longrightarrow\mu=\frac{n^2}{2}-\frac{n^2-n}{\pi}arctan\sqrt{\frac{1+\rho^2}{3+\rho^2}},\quad \mathrm{as}\ M\rightarrow\infty,
\end{equation}
\begin{equation}
\hat{\rho_M}=\sqrt{\frac{3-tan^2\pi T_M}{tan^2\pi T_M-1}}\stackrel{\text{P}}\longrightarrow \rho,\quad \mathrm{as}\ M\rightarrow\infty.
\end{equation}
In conclusion, $\hat{\rho_M}$ is a consistent estimator of $\rho$.
\end{proof}

\section*{\centering{Acknowledgements}}

We are grateful to Prof. W. Huang and Prof. Q.-H. Zhang for many useful discussions and suggestsions on systemizing our ideas and polishing this thesis.


\begin{thebibliography}{4}

\bibitem{ALA}
\textsc{Ailon, N.} (2012). 
An active learning algorithm for ranking from pairwise preferences with an almost optimal query complexity. 
\textit{JMLR.org}

\bibitem{ERS}
\textsc{Aldous, D.} (2017). 
Elo ratings and the sports model: a neglected topic in applied probability? \textit{Statistical Science} 
\textbf{32} 616-629

\bibitem{RCDS}
\textsc{Ammar, A. and Shah, D.} (2012). 
Ranking: Compare, don't score. 
\textit{Communication, Control, and Computing. IEEE} 776-783.

\bibitem{}
\textsc{Bradley, R. A.} (1976). 
Science, statistics, and paired comparisons
\textit{Biometrics}
\textbf{32}(2) 213-239.

\bibitem{RAI}
\textsc{Bradley, R. A. and Terry, M. E.} (1952). 
Rank analysis of incomplete block designs: I. The method of paired comparisons. 
\textit{Biometrika}
\textbf{39} 324-345.

\bibitem{}
\textsc{Cattelan, M.} (2012). 
Models for paired comparison data: A review with emphasis on dependent data. 
\textit{Statistical Science} 
\textbf{27} 412-433.

\bibitem{DBTM}
\textsc{Cattelan, M., Varin, C. and Firth, D.} (2013). 
Dynamic Bradley-Terry modelling of sports tournaments. 
\textit{Journal of the Royal Statistical Society Series C} 
\textbf{62} 135-150.

\bibitem{RCP}
\textsc{\'{E}l\H{o}, A. I.} (2008). 
The rating of chessplayers, past and present.
\textit{Ishi Press}

\bibitem{DSM}
\textsc{Fahrmeir, L. and Tutz, G.} (1994). 
Dynamic Stochastic Models for Time-Dependent Ordered Paired Comparison Systems.
\textit{Journal of the American Statistical Association}
\textbf{89} 1438-1449.

\bibitem{TPCETM}
\textsc{Glenn, W. A. and David, H. A.} (1960). 
Ties in paired-comparison experiments using a modified Thurstone-Mosteller model.
\textit{Biometrics}
\textbf{16} 86-109.

\bibitem{DPC}
\textsc{Glickman, M. E.} (2001). 
Dynamic Paired Comparison Models with Stochastic Variances. 
\textit{Journal of the American Statistical Association} 
\textbf{28} 673-689.

\bibitem{PEL}
\textsc{Glickman, M. E.} (1999). 
Parameter estimation in large dynamic paired comparison experiments. 
\textit{Journal of the Royal Statistical Society Series C} 
\textbf{48} 377-394.

\bibitem{SMN}
\textsc{Glickman, M. E. and Stern, H. S.} (1998). 
A state-space model for National Football League scores. 
\textit{Journal of the Americanm Statistical Association}  
\textbf{93} 25-35.

\bibitem{}
\textsc{Guttman, L.} (1946). 
An Approach for Quantifying Paired Comparisons and Rank Order. 
\textit{The Annals of Mathematical Statistics}
\textbf{17}(2) 144-163.

\bibitem{MMA}
\textsc{Hunter, D. R.} (2004). 
MM Algorithms for Generalized Bradley-Terry Models. 
\textit{Annals of Statistics}
\textbf{32}(1) 384-406.

\bibitem{}
\textsc{Joe, H.} (1988). 
Majorization, Entropy and Paired Comparisons. 
\textit{Annals of Statistics}
\textbf{16}(2) 915-925.

\bibitem{}
\textsc{Kendall, M. and Smith, B.} (1940). 
On the Method of Paired Comparisons. 
\textit{Biometrika}
\textbf{31}(3/4) 324-345.

\bibitem{MCS}
\textsc{Kir\'{a}ly, F. J. and Qian, Z. Z.}(2017).
Modelling Competitive Sports: Bradley-Terry-\'{E}l\H{o} Models for Supervised and On-Line Learning of Paired Competition Outcomes.
\textit{arXiv}
\textbf{1701.08055}

\bibitem{}
\textsc{Knorr-Held, L.} (2000).
Dynamic rating of sports teams. 
\textit{The Statistician}
\textbf{49}  261-276.

\bibitem{SGT}
\textsc{Kovalchik, S.} (2016). 
Searching for the GOAT of tennis win prediction. 
\textit{J. Quant. Anal. Sports} 
\textbf{12} 127-138.

\bibitem{}
\textsc{Mosteller, F.}(1951).
Remarks on the method of paired comparisons: I.The least squares solution assuming equal standard deviations and equal correlations.
\textit{Psychometrika}
\textbf{16}(1) 3-9.

\bibitem{}
\textsc{Mosteller, F.}(1951).
Remarks on the method of paired comparisons: II.The effect of an aberrant standard deviation when equal standard deviations and equal correlations are assumed.
\textit{Psychometrika}
\textbf{16}(2) 203-206.

\bibitem{}
\textsc{Mosteller, F.}(1951).
Remarks on the method of paired comparisons: III.A test of significance for paired comparisons when equal standard deviations and equal correlations are assumed.
\textit{Psychometrika}
\textbf{16}(2) 207-218.

\bibitem{IRP}
\textsc{Negahban, S., Oh, S. and Shah D.} (2012). 
Iterative Ranking from Pair-wise Comparisons.
\textit{Advances in Neural Information Processing Systems}
\textbf{3}(93) 2483-2491.

\bibitem{TPCEBT}
\textsc{Rao, P. V. and Kupper, L. L.} (1967). 
Ties in Paired-Comparison Experiments: A Generalization of the Bradley-Terry Model.
\textit{Journal of the American Statistical Association} 
\textbf{62} 194-204.

\bibitem{SPSL}
\textsc{Robbins, H.} (1968). 
A Sequential Procedure for Selecting the Largest of $k$ Means. 
\textit{The Annals of Mathematical Statistics} 
\textbf{39}(1) 88-92.

\bibitem{SSMD}
\textsc{Robert, E. B.}(1954).
A Single-Sample Multiple Decision Procedure for Ranking Means of Normal Populationswith known Variances.
\textit{The Annals of Mathematical Statistics}
\textbf{25}(1) 16-39.

\bibitem{}
\textsc{Shah, N. B., Balakrishnan, S. and Bradley, J., et al.} (2016). 
Estimation from pairwise comparisons: sharp minimax bounds with topology dependence. \textit{Journal of Machine Learning Research}
\textbf{17}(1) 2049-2095.

\bibitem{}
\textsc{Shah, N. B., Balakrishnan, S. and Guntuboyina, A., et al.} (2017). Stochastically Transitive Models for Pairwise Comparisons: Statistical and Computational Issues.
\textit{IEEE Transactions on Information Theory}
\textbf{63}(2) 934-959.

\bibitem{PA}
\textsc{Thurstone, L. L.}(1927).
Psychophysical Analysis.
\textit{The American Journal of Psychology}
\textbf{38}(3) 368-389.

\bibitem{}
\textsc{Thurstone, L. L.} (1994).
A law of Comparative Judgment. 
\textit{Psychological Review} 
\textbf{101}(2) 266-270.

\bibitem{AS}
\textsc{Tocher, K. D.} (1963).
The Art of Simulation. 
\textit{English University Press, London}

\bibitem{DBT}
\textsc{Zermelo, E.} (1929).
Die Berechnung der Turnier-Ergebnisse als ein Maximumproblem der Wahrscheinlichkeitsrechnung.
\textit{Mathematische Zeitschrift}
\textbf{29} 436-460.

\bibitem{hearth1}
\textsc{Gamepedia} (2018).
Arena, Hearthstone Wiki, Available at \url{https://hearthstone.gamepedia.com/Arena}.

\bibitem{hearth2}
\textsc{Gamepedia} (2018).
Your Standard Brawliseum, Hearthstone Wiki, Available at \url{https://hearthstone.gamepedia.com/Your\_Standard\_Brawliseum}.

\end{thebibliography}
\end{document}